%% file: paper.tex
\documentclass[journal]{IEEEtran}

\input{src/usedPackages}

\author{Max Spahn, Martijn Wisse, Javier Alonso Mora}

\title{\LARGE \bf
Dynamic Optimization Fabrics for Motion Generation
}

\author{Max Spahn$^1$, Martijn Wisse$^1$, Javier Alonso-Mora$^1$
\thanks{*This research was supported by Ahold Delhaize. All content represents the opinion of the author(s), which is not necessarily shared or endorsed by their respective employers and/or sponsors.}
\thanks{$^1$The authors are with the Department of Cognitive Robotics, Delft University of Technology, 2628 CD, Delft, The Netherlands
       {\tt\small \{m.spahn, m.wisse, j.alonsomora\}@tudelft.nl}}
}

\date{April 2021}

\graphicspath{{res/}}

\newif\iftrackchanges

\begin{document}
\trackchangesfalse

\input{src/commands}
\maketitle

\input{src/abstract}

\input{src/intro}
\input{src/state}

\input{src/math}
\input{src/methods}

\input{src/extension_non_holonomic}

\input{src/results}

\input{src/conclusion}

\bibliographystyle{IEEEtran}
\bibliography{lib/clean.bib,lib/external.bib}
\newpage
\input{src/bibio}

\end{document}

%% file: src/usedPackages.tex
\usepackage[utf8]{inputenc}
\usepackage{bm}
\usepackage{amsmath}
\usepackage{amsthm}
\usepackage{parskip}
\usepackage{paralist}
\usepackage{graphicx}
\usepackage{subcaption}
\usepackage{hyperref}
\usepackage{xcolor}
\usepackage{amsfonts}
\usepackage{acronym}
\usepackage{tcolorbox}
\usepackage[normalem]{ulem} 
\usepackage[ruled]{algorithm2e}
\usepackage[numbers]{natbib}
\usepackage[capitalise]{cleveref}

%% file: src/commands.tex
\newcommand\X{\ensuremath{\mathcal{X}}}
\newcommand\Xr{\ensuremath{\mathcal{X}_{\textrm{rel}}}}
\newcommand\Xj{\ensuremath{\mathcal{X}_j}}
\newcommand\Q{\ensuremath{\mathcal{Q}}}
\newcommand\Rn{\ensuremath{\mathbb{R}^n}}
\newcommand\Rmj{\ensuremath{\mathbb{R}^{m_j}}}
\renewcommand\l{\ensuremath{\mathcal{L}}}
\renewcommand\le{\ensuremath{\mathcal{L}_e}}
\newcommand\ld{\ensuremath{\mathcal{L}_d}}
\renewcommand\lg{\ensuremath{\mathcal{L}_g}}
\newcommand\he{\ensuremath{\mathcal{H}_e}}
\newcommand\hd{\ensuremath{\mathcal{H}_d}}
\newcommand\dt{\ensuremath{\Delta t}}
\newcommand\Me{\ensuremath{\mat{M}_{\le}}}
\newcommand\fe{\ensuremath{\vec{f}_{\le}}}
\newcommand\Pe{\ensuremath{\mat{P}_{\le}}}
\newcommand\M{\ensuremath{\mat{M}}}
\newcommand\Mnh{\ensuremath{\mat{M}_{\textrm{nh}}}}
\newcommand\I{\ensuremath{\mat{I}}}
\newcommand\f{\ensuremath{\vec{f}}}
\newcommand\fnh{\ensuremath{\vec{f}_{\textrm{nh}}}}
\newcommand\h{\ensuremath{\vec{h}}}
\newcommand\zerovec{\ensuremath{\vec{0}}}
\newcommand\Spec{\ensuremath{\mathcal{S}}}
\newcommand\htwo{\ensuremath{\vec{h}_2}}
\newcommand\Md{\ensuremath{\mat{M}_d}}
\newcommand\fd{\ensuremath{\vec{f}_d}}
\newcommand\Mde{\ensuremath{\mat{M}_{de}}}
\newcommand\fde{\ensuremath{\vec{f}_{de}}}
\newcommand\forc{\ensuremath{\vec{\psi}}}
\newcommand\spec{\ensuremath{\left(\M,\f\right)_{\X}}}
\newcommand\x{\ensuremath{\vec{x}}}
\newcommand\xdot{\ensuremath{\dot{\x}}}
\newcommand\xddot{\ensuremath{\ddot{\x}}}
\newcommand\xj{\ensuremath{\vec{x}_j}}
\newcommand\xjdot{\ensuremath{\dot{\xj}}}
\newcommand\xjddot{\ensuremath{\ddot{\xj}}}
\newcommand\xb{\ensuremath{\bar{\vec{x}}}}
\newcommand\xbdot{\ensuremath{\dot{\xb}}}
\newcommand\xbddot{\ensuremath{\ddot{\xb}}}
\newcommand\xt{\ensuremath{\vec{\tilde{x}}}}
\newcommand\xtdot{\ensuremath{\dot{\xt}}}
\newcommand\xtddot{\ensuremath{\ddot{\xt}}}
\newcommand\xr{\ensuremath{\x_{\textrm{rel}}}}
\newcommand\xrdot{\ensuremath{\dot{\x}_{\textrm{rel}}}}
\newcommand\xrddot{\ensuremath{\ddot{\x}_{\textrm{rel}}}}
\newcommand\q{\ensuremath{\vec{q}}}
\newcommand\qdot{\ensuremath{\dot{\q}}}
\newcommand\qddot{\ensuremath{\ddot{\q}}}
\newcommand\qt{\ensuremath{\vec{\tilde{q}}}}
\newcommand\qtdot{\ensuremath{\dot{\qt}}}
\newcommand\qtddot{\ensuremath{\ddot{\qt}}}

\newcommand\J{\ensuremath{\mat{J}_{\phi}}}
\newcommand\Jt{\ensuremath{\mat{J}^T_{\phi}}}
\newcommand\Jdot{\ensuremath{\dot{\mat{J}}_{\phi}}}
\newcommand\Jnh{\ensuremath{\mat{J}_{\textrm{nh}}}}
\newcommand\Jnht{\ensuremath{\mat{J}^T_{\textrm{nh}}}}
\newcommand\Jnhdot{\ensuremath{\dot{\mat{J}}_{\textrm{nh}}}}
\newcommand{\map}{\ensuremath{\phi}}
\newcommand{\mapt}{\ensuremath{\phi_t(\vec{q})}}
\newcommand{\mapd}{\ensuremath{\phi_d}}
\newcommand{\g}{\ensuremath{\vec{g}(t)}}
\newcommand{\gp}{\ensuremath{\vec{g}^{\prime}(t)}}
\newcommand{\gpp}{\ensuremath{\vec{g}^{\prime\prime}(t)}}
\newcommand{\pull}[2]{\textrm{pull}_{#1}{#2}}
\newcommand{\energize}[2]{\textrm{energize}_{#1}{#2}}
\newcommand{\der}[2]{\partial_{#1}#2}
\newcommand{\dertwo}[2]{\partial^2_{#1}#2}
\newcommand{\derf}[2]{\frac{\partial#2}{\partial#1}}
\newcommand{\derftwo}[3]{\frac{\partial^2#3}{\partial#1\partial#2}}
\newcommand{\dert}[1]{\frac{d}{dt}#1}
\newcommand{\pinv}[1]{#1^{\dagger}}
\newcommand{\norm}[1]{\left\lVert#1\right\rVert}
\newcommand\mat[1]{\ensuremath{\bm{#1}}}
\renewcommand\vec[1]{\ensuremath{\bm{#1}}}

\newtheoremstyle{spaced}
  {1em plus .2em minus .0em}
  {1em plus .2em minus .0em}
  {\itshape}
  {}
  {\bfseries}
  {.}
  {0.5em}
  {}

\newtheorem{definition}{Definition}[section]
\newtheorem{theorem}[definition]{Theorem}
\newtheorem{proposition}[definition]{Proposition}
\newtheorem{lemma}[definition]{Lemma}

\newcommand{\panda}{panda robot}

\definecolor{OliveGreen}{rgb}{0,0.6,0}
\newcommand{\MS}[1]{\textbf{{\color{OliveGreen}{{#1}}}}} 
\iftrackchanges \newcommand{\changed}[2][]{{\color{blue}{#2}}} \else \newcommand{\changed}[2][]{#2} \fi
\iftrackchanges \newcommand{\deleted}[2][]{} \else \newcommand{\deleted}[2][]{} \fi

\acrodef{mpc}[MPC]{Model Predictive Control}
\acrodef{sf}[SF]{Static Fabrics}
\acrodef{df}[DF]{Dynamic Fabrics}

\renewcommand{\quote}[1]{\textbf{Remark:} \textit{#1}}
\newcommand{\response}[1]{\textbf{Response:} #1}
\newcommand{\todo}[1]{{\color{red}{#1}}}
\newenvironment{fromtext}[1]
  {\textsc{#1}\begin{center}
    \begin{tabular}{|p{0.9\textwidth}|}
    \hline\\
    }
    { 
    \\\hline
    \end{tabular} 
    \end{center}
  }

%% file: src/abstract.tex
\begin{abstract}

Optimization fabrics are a geometric approach to real-time local
motion generation, where motions are designed by the composition
of several differential equations that exhibit a desired motion behavior.
We generalize this framework to dynamic scenarios and non-holonomic robots
and prove that fundamental properties can be conserved. We show that convergence to
desired trajectories and avoidance of moving obstacles can be guaranteed using
simple construction rules of the components. Additionally, we present the
first quantitative comparisons between optimization fabrics and
model predictive control and show that optimization fabrics can generate
similar trajectories with better scalability, and thus, much
higher replanning frequency (up to 500 Hz with a 7 degrees of
freedom robotic arm). Finally, we present empirical results on
several robots, including a non-holonomic mobile manipulator with 10
degrees of freedom and avoidance of a moving human,
supporting the theoretical findings.

\changed{The open-source implementation can be found at
  \url{https://github.com/tud-amr/fabrics}
}
\end{abstract}

%% file: src/intro.tex
\section{Introduction}%
\label{sec:introduction}

Robots increasingly populate dynamic environments. 
Imagine a robot
operating alongside customers in a supermarket. It is requested to perform different
tasks, such as cleaning the floor or picking a wide range of products.
These different manipulation tasks may vary in their dimension
and accuracy requirements, e.g. rotation around a suction gripper
does not need to be specified while two-finger grippers require full poses. Thus,
it is important for motion planning algorithms to support various goal definitions.
Further, the robot is operating alongside humans, it has to constantly react to
the changing environment and consequently update an initial plan.
As customers move fast, the adaptations must be
computed in real time.
Therefore, motion planning is often divided into global motion
planning~\cite{Karaman2011} and local motion planning, which we will refer
to as motion generation in this paper. A global planner
generates a first feasible path that is used by a motion generator as global guidance.
This paper proposes a novel approach to motion generation, that deals
with a variety of different goal definitions. 

Motion generation is often solved by formulating an optimization problem 
over a time horizon. The popularity of this approach is partly thanks to
the guaranteed collision avoidance and thus safety~\cite{Hrovat2012,Hewing2020}.
The optimization problem is then assembled from a scalar objective
function, encoding the motion planning problem (e.g., the desired final position, path
constraints, etc.), the transition function, defining the robot’s
dynamics, and several inequality constraints, integrating physical limits and
obstacle avoidance. 
\begin{figure}[t]
  \centering
  \begin{subfigure}{0.5\linewidth}
    \centering
    \includegraphics[width=0.9\textwidth]{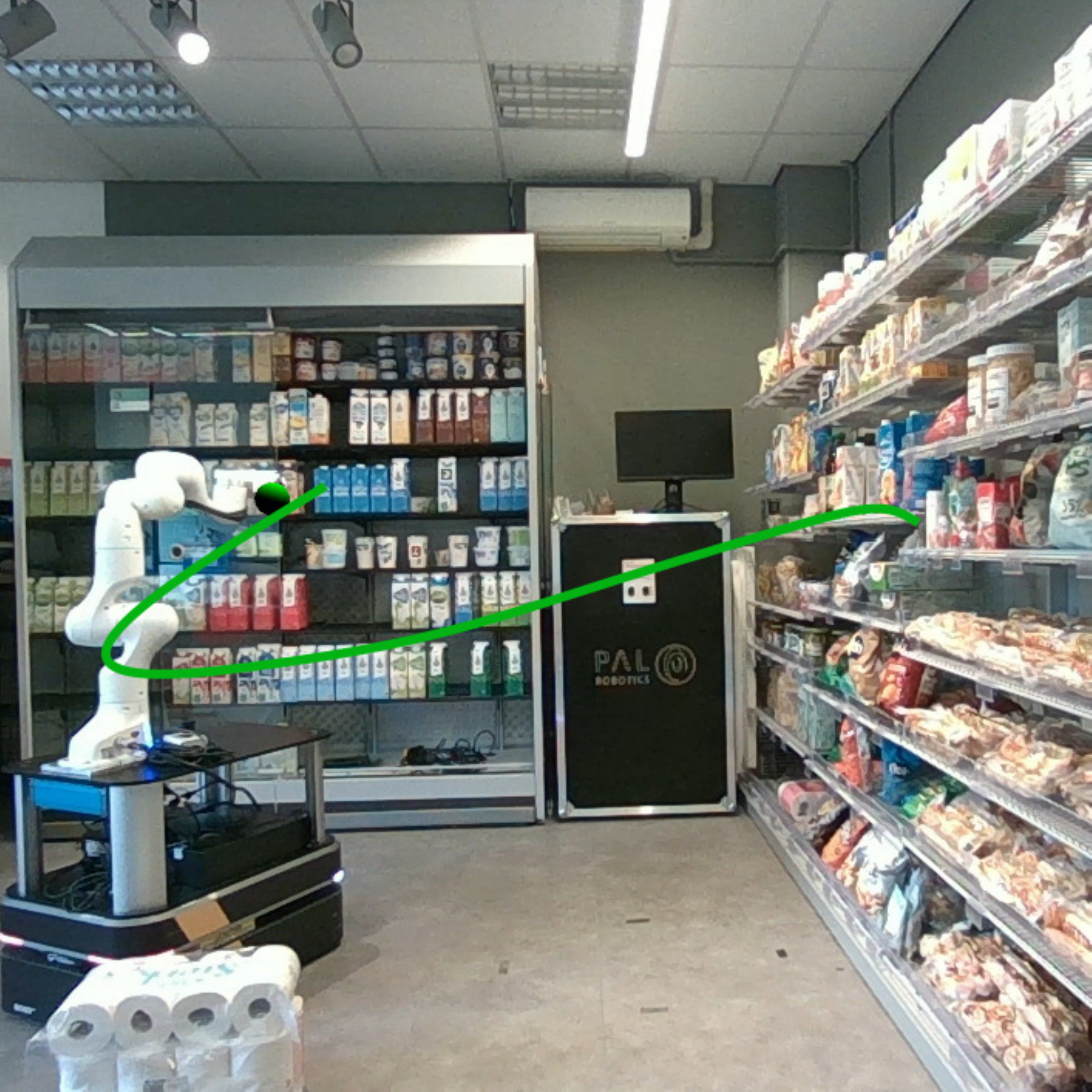}
    \caption{}
    \label{subfig:albert_spline_1}
  \end{subfigure}%
  \begin{subfigure}{0.5\linewidth}
    \centering
    \includegraphics[width=0.9\textwidth]{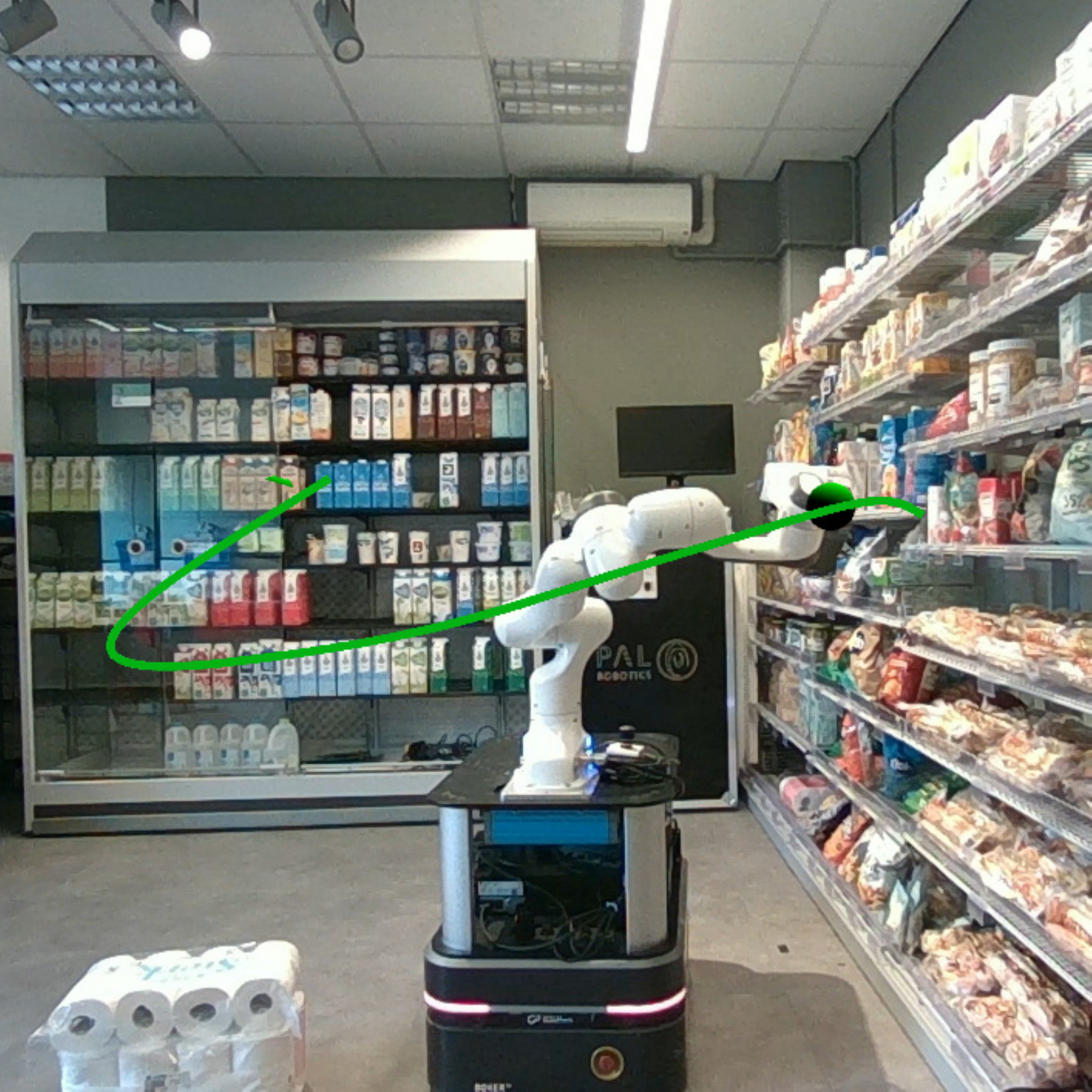}
    \caption{}
    \label{subfig:albert_spline_2}
  \end{subfigure}%
  \caption{
    Dynamic fabrics for path (green) following with a non-holonomic mobile
    manipulator. Dynamic fabrics control all actuators simultaneously to follow
    the end-effector path while keeping a given orientation and avoiding collision
    with the environment.
  }
  \label{fig:albert_spline_example}
\end{figure}
%
Despite abundant applications of such optimization-based approaches
to mobile robots, the computational costs limit applicability when dealing with
high-dimensional configuration spaces~\cite{Bednarczyk2020,Richter2010}.
Data-driven approaches to speed up the optimization process usually come with
reduced generalization abilities, loss of formal guarantees~\cite{Hewing2020}
and require prior, often
costly, data acquisition. Moreover, due to the scalar objective function, the
user must carefully weigh up different parts of the objective function. As
a consequence, optimization-based approaches are challenging to tune and
inflexible to generic motion planning problems with variable goal
objectives~\cite{Xie2020,Edwards2021}. 

In the field of geometric control, namely Riemannian
motion policies (RMP) and optimization fabrics, all individual parts of the motion planning
problem are formulated as differential equations of second order.
Applying operations from differential geometry, the individual
components are combined in the configuration space to define the
resulting motion~\cite{Cheng2018,Ratliff2020}. This allows to iteratively
\textit{design} the motion of the robot while maintaining explainability over the resulting
motion~\cite{Cheng2018,Ratliff2020,Xie2020,Wyk2022}.

These works on optimization fabrics~\cite{Ratliff2020}, but also on
predecessors, such as RMP~\cite{Cheng2018} and RMP-Flow~\cite{Cheng2020a},
have shown the power of designing reactive behavior as second-order
differential equations. However, integration of dynamic features, such as
moving obstacles and path following, have not been proposed
nor have the framework been applied to non-holonomic systems.
In this article, we exploit relative coordinate systems in the framework of
optimization fabrics by introducing the dynamic pullback operation (\cref{eq:dynamic_pull}).
This generalization 
can then integrate moving obstacles and path following.
We show that our generalization maintains
guaranteed convergence for path following tasks and improves collision
avoidance with moving obstacles.
Moreover, we propose a method
to incorporate non-holonomic constraints.
Lastly, we compare a trajectory optimization formulation, namely a model
predictive control formulation, with optimization fabrics to provide the reader
with a better understanding of key differences between the two approaches.
We analyze computational costs and the quality of resulting trajectories
for different robots.
Several simulated results and real-world experiments
show the practical implications of \ac{df}. The contributions of this paper can be
summarized as:


\begin{enumerate}
  \item We enable the usage of optimization fabrics for dynamic scenarios.
    Specifically, we propose time parameterized differential maps using up-to
    second-order predictor models. As a consequence, this enables the
    integration of moving obstacles and path following tasks.
  \item We extend the framework of optimization fabrics to non-holonomic robots.
  \item We present a quantitative comparison between model predictive control and
  optimization fabrics. The results reveal that fabrics are an order of magnitude faster,
  more reliable, and easier to tune for goal-reaching tasks with a robotic manipulator in
  static environments.
\end{enumerate}
All findings are supported by extensive experiments in both simulation and real-world
with a manipulator, a differential drive robot, and a mobile manipulator.

%% file: src/state.tex
\section{Related Work}%
\label{sec:state}

In dynamic environments, global planning methods are not sufficient
due to low planning frequencies. Thus, local motion generation methods, like
the one presented in this work are employed. These methods typically
require guidance to avoid local minima and thus effectively solve planning problems.

\subsection{Task constrained global motion planning}%
\label{sub:state_global}

Motion planning problems are usually defined by goals in arbitrary task spaces,
such as the 3D Euclidean space or end-effector poses. In this
context, tasks can be regarded as constraints to the motion planning problem.
Conventional approaches to motion planning rely on inverse kinematics to
transform task constraints into sets of configurations. The resulting global
motion planning problem is then often solved using sampling-based methods
\cite{Rickert2014}.

Sampling-based motion planners generate random configurations
until a valid path between an initial
configuration and a set of goal configurations is found
\cite{Karaman2011}. Several methods have been proposed to directly
integrate task constraints into the sampling phase. 
\cite{Stilman2010} proposed a
method to iteratively push a random sample to the manifold adhering to the task constraint.
The notion of
task constraints was later extended to task space regions to define soft constraints for
individual task components~\cite{Berenson2011}.
\cite{Kingston2019} proposed scalar-valued functions to represent task
constraints for sampling-based planning.
As all of the above-mentioned methods rely on implicitly constrained sampling
in the joint space, they exhibit high computational time, which is especially
harmful to real-world applications \cite{Qureshi2020}
and require local motion generation methods for path following
and execution in dynamic environments
\cite{Brito2019}.
In the next
subsection, recent developments in local motion planning are summarized.

\subsection{Receding-horizon trajectory optimization}%
\label{sub:state_local}

Methods formulating motion generation as an optimization problem with a finite
discrete time horizon are known under the name of receding-horizon trajectory
optimization. In line with most literature in robotics, we will refer to such
methods as \ac{mpc}.
Generally, several objectives are encoded in the scalar cost function, dynamics
are formulated as equality constraints and inequality constraints ensure
collision avoidance and joint limit avoidance. The dynamics for this problem can
include the full dynamics model or simple integrating schemes \cite{Hewing2020}.
By explicitly solving the constrained optimization problem, this approach yields
formal guarantees on stability. Stability for \ac{mpc} is proven by formulating
an appropriate Lyapunov function and showing that the finite time-horizon
formulation with an appropriate terminal cost results in the same stability as
the corresponding infinite time-horizon formulation
\cite{l1,l4,keerthi1988optimal}. \ac{mpc} has been applied to various robotic
systems in dynamic environments, such as drones \cite{Tordesillas2019}, mobile
robots \cite{Brito2019}, and mobile manipulators
\cite{Avanzini2015,Avanzini2018}. Despite these results, formal stability
guarantees in such environments are challenging as appropriate terminal cost
functions are often not computable or too conservative. Besides, the
computational costs scale with the degrees of freedom restricting real-time
applicability to simple dynamics and environment models \cite{Spahn2021}.

Some \ac{mpc} formulations are non-linear and can be analyzed
using methods from non-linear control. When analyzing non-linear control
system, Riemannian energies lead to more detailed stability results than
Lyapunov functions. By investigating the variation around the generated
trajectory and its contracting towards the desired trajectory, some control
designs show exponential stabilizing properties \cite{l2}. These
findings have been applied to tracking control problems \cite{l3}.

\subsection{Riemannian motion policies and fabrics}%
\label{sub:riemannian_motion_policies_and_fabrics}

Based on the findings of contracting metrics for non-linear
control design \cite{l2,l3}, geometric control approaches design the motion
generation such that convergence is inherent to the problem formulation
rather than imposing them on the solution process. Practically, individual
constraints to the motion planning problem shape the optimization manifold
so that the solution is accessible through the solution of simple differential
equation. An example for shaping the optimization
manifold is seen in \cref{fig:spec_combination}.

\begin{figure}[h]
  \centering
  \begin{subfigure}{0.33\linewidth}
    \centering
    \includegraphics[width=0.9\textwidth]{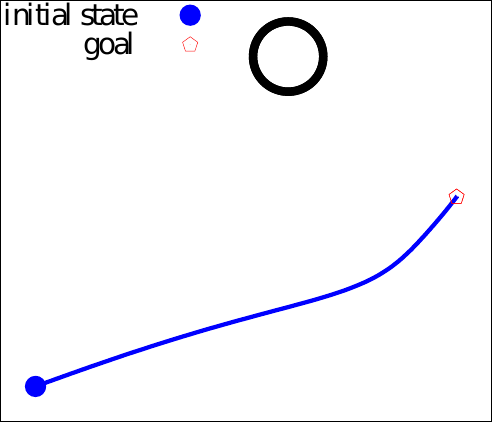}
    \caption{}
    \label{subfig:trajectory_obst1}
  \end{subfigure}%
  \begin{subfigure}{0.33\linewidth}
    \centering
    \includegraphics[width=0.9\textwidth]{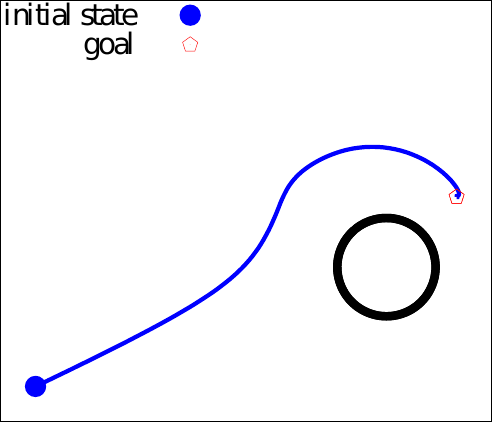}
    \caption{}
    \label{subfig:trajectory_obst2}
  \end{subfigure}%
  \begin{subfigure}{0.33\linewidth}
    \centering
    \includegraphics[width=0.9\textwidth]{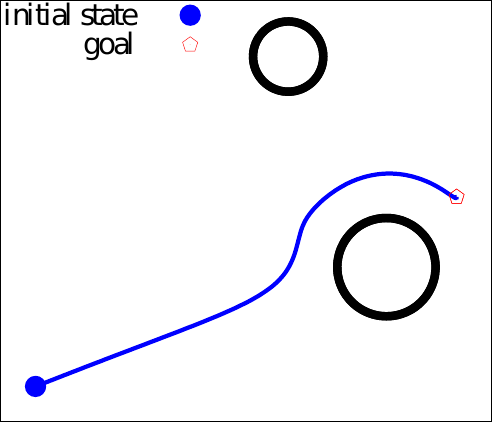}
    \caption{}
    \label{subfig:trajectory_both_obstacles}
  \end{subfigure}
  \caption{
    Combining different avoidance behaviors using optimization fabrics. The
    components defining collision avoidance with single obstacles (a,b) are
    combined in (c). Obstacles are shown in black. Trajectories of the point
    robot are shown in blue.
  }
  \label{fig:spec_combination}
\end{figure}
Realizing this concept, Riemannian motion policies (RMPs) represent a
natural way of combining multiple policies into one joined policy.
RMPs define individual sub-tasks of the motion planning as
differential equations (\textit{spectral semi sprays} or
\textit{specs} for short) of second order and propose the
\textit{pullback} and \textit{summation} operators to combine multiple policies
in the configuration space. As subtasks can be defined in arbitrary manifolds
of the configuration space, RMP generalize operational space
control~\cite{Khatib1987}. The resulting behavior of RMP was reported to be
intuitive while keeping computational costs low~\cite{Ratliff2018}. The concept
of RMP was used in~\cite{Cheng2018,Cheng2020a} to form RMP-Flow, a motion
planning algorithm that is shown to be conditionally stable and invariant
across robots. In RMP-Flow, individual tasks are represented as a pair of
a motion policy and a corresponding metric defining the importance of
individual directions. An RMP adaptation was proposed for non-holonomic robots
in~\cite{Meng2019}. By incorporating the kinematic constraint into the root
equation of the RMP, the computed policy is applicable to non-holonomic robots.
Besides, that work proposed a neural net to learn the collision avoidance task
components. 

Although RMPs have proven to be a powerful tool for
motion generation, it was reported to require intuition and experience
during tuning~\cite{Ratliff2020}. Optimization fabrics with Finsler structures
as metric generators simplify the motion design as the conditions for stability
and convergence are inherent to the definition of Finsler
structures~\cite{Ratliff2020,Ratliff2021}. Opposed to RMPs, where the metric is
typically user-defined, fabrics derive Finsler metrics from artificial
energies, similar to approaches from control design, \cite{l2,l3}, using the
Euler-Lagrange-Equation from geometric mechanics. Although fabrics generalize
the concept of RMPs and make it accessible to a broader audience by
decreasing the intuition and expertise required, they have not yet been applied
to a wide range of robots. 

The reason for this lack of application of fabrics is twofold. First, all the
above mentioned methods are reactive and highly local methods, thus making them
prone to local minima \cite{Bhardwaj2021}. As RMPs and optimization
fabrics do not incorporate path following, integration of global
planning to overcome local minima is not possible to this date. Second, fabrics
and RMP do not make use of velocity estimates of obstacles but rely purely on
their high reactivity in dynamic environments. As for other trajectory
optimization techniques, motion estimates could benefit fabrics (and RMP) to
result in even smoother motion and allow applications in such environments. 

In this paper, we address these issues by proposing time parameterized
differential maps to form \changed{\acl{df}}. This generalization integrates
path following and velocity estimates of moving
obstacles. Together with the extension to non-holonomic robots, our method
allows to deploy the promising theory of optimization fabrics to mobile
manipulators, operating in dynamic environments.

%% file: src/math.tex
\section{Background}%
\label{sec:mathematics}

In the previous section, we have highlighted that optimization fabrics represent a
powerful tool for reactive motion generation. Since we generalize this concept, this
section aims at familiarizing the reader with key findings on optimization fabrics and
recalling some of the basic notations known from differential geometry. 
We first give an overview on how optimization fabrics are used for motion generation
and how the components are derived and composed. Then the theoretical foundations
are summarized from \cite{Cheng2020,Ratliff2020}.

\subsection{Motion generation using optimization fabrics}%
\label{sub:trajectory_generation_using_optimization_fabrics}

When using optimization fabrics for motion generation, all components including constraints and 
goal attraction are designed as second order differential equations. If specific design
rules for these equations are respected, all components can be combined to 
form a converging motion generator. Specifically, the following steps are performed:
\begin{enumerate}
  \item Design path-consistent geometries in suited manifolds of the configuration space
    (\cref{eq:homogeneous}).
  \item Design corresponding Finsler energies defining the importance metric in this manifold
    (\cref{sub:conservative_fabrics_and_energization}).
  \item Energize all geometries with the associated Finsler energies 
    (\cref{sub:conservative_fabrics_and_energization}).
  \item Pull back the energized systems into the configuration space and sum them 
    (\cref{sub:operations_on_specs}).
  \item Force the combined system with a differentiable potential. As a composition of optimization fabrics, 
    the resulting trajectory converges towards the potential's minimum
    (\cref{sub:optimization_fabrics}).
\end{enumerate}
In the following, we introduce the reader to the theory of optimization fabrics
and recall important findings from \cite{Ratliff2020}.

\subsection{Configurations and task variables}%
\label{sub:configurations_and_task_variables}

We denote $\q\in\Q\subset\Rn$ a
configuration of the robot with $n$ its degrees of freedom; \Q{} is the configuration space of the generalized coordinates
of the system. Generally, $\q(t)$ defines the robot's configuration at time $t$, so that 
\qdot, \qddot{} define the instantaneous derivatives of the robot's configuration.
Similarly, we assume
that there is a set of task variables $\xj\in\Xj\subset\Rmj$ with variable dimension
$m_j \leq n$. The task manifold \Xj{} defines an arbitrary manifold of the configuration
space \Q{} in which a robotic task can be represented. 
Further, we assume that there is a differential map
$\map_j:\Rn\rightarrow\Rmj$ that relates the configuration space to the $j^{th}$ task
space. For example, when a task variable is defined as the end-effector position, then
$\map_j$ is the positional part of the forward kinematics. On the other hand, if a task
variable is defined to be the joint position, then $\map_j$ is the identity function. 
In the following, we drop the subscript $j$ in most cases for readability when the
context is clear.

In this work, we assume that \map{} is smooth and twice differentiable so that the Jacobian is
defined as
\begin{equation}
  \J = \derf{\q}{\map} \in \mathcal{R}^{m\times n}, 
\end{equation}
or $\J = \der{\q}{\map}$ for short.
Thus, we can write the total time derivatives of \x{} as
\begin{align}
  \xdot &= \J\qdot \\
  \xddot &= \J\qddot + \Jdot\qdot.
\end{align}

\subsection{Spectral semi-sprays}%
\label{sub:spectral_semi_sprays}

Inspired by simple mechanics (e.g., the simple pendulum), the framework of optimization
fabrics designs motion generation as second-order dynamical
systems $\xddot = \pi(\x,\xdot)$~\cite{Cheng2020,Ratliff2020}. While higher-order systems seem feasible, their
implementation on robots is much more challenging, as higher order configuration space
derivatives would be required. The trajectory generator is defined by the differential equation
$\M\xddot + \f = 0$, where $\M(\x,\xdot)$ and $\f(\x,\xdot)$ are functions of position and
velocity. Besides, \M{} is symmetric and invertible. Such systems $\Spec = \spec$ are
known as \textit{spectral semi-sprays}, or \textit{specs} for short.  When the space of
the task variable is clear from the context, we drop the subscript.  Then, the trajectory
is computed as the solution to the system $\xddot=-\M^{-1}\f$.

\subsection{Operations on specs}%
\label{sub:operations_on_specs}
Next, the two fundamental operations for specs, transformation between spaces and
summation, are introduced.

Given a differential map $\map: \Q\rightarrow\X$ and a spec \spec{}, the \textit{pullback}
is defined as 
\begin{equation}
  \pull{\map}{\spec} = {\left(\Jt\M\J, \Jt(\f+\M\Jdot\qdot)\right)}_{Q}.
\end{equation}
The pullback allows converting between two distinct manifolds (e.g. a spec could be 
defined in the robot's workspace and being pulled into the robot's configuration space using
the pullback with \map{} being the forward kinematics).

For two specs, $\Spec_1 = {\left(\M_1,\f_1\right)}_{\X}$ and 
$\Spec_2 = {\left(\M_2,\f_2\right)}_{\X}$, their \textit{summation} is defined by:
\begin{equation}
  \Spec_1 + \Spec_2 = {\left(\M_1 + \M_2, \f_1 + \f_2\right)}_{\X}.
\end{equation}

\subsection{Optimization fabrics}%
\label{sub:optimization_fabrics}

Optimization fabrics form a special class of specs, and thus they inherit their properties,
specifically the previously defined operations of \textit{summation} and \textit{pullback}.
First, let us introduce a finite and differentiable potential
function $\forc(\x)$ defined in a task manifold \X{}. 
Then, the modified spec $\Spec_{\forc} = \left(\M,\f + \der{\x}{\forc}\right)$
is called the \textit{forced variant} of $\Spec = \spec$.
Only if the trajectory $\x(t)$ generated by the forced spec converges to the minimum of \forc{}, 
the spec is said to form an \textit{optimization fabric}.
When the spec only converges to the minimum when equipped with a damping term,
$\left(\M,\f + \der{\x}{\forc} + \mat{B}\xdot\right)$, 
it forms a \textit{frictionless fabric}~\cite[Definition 4.4]{Ratliff2020}. 
Note that the mechanical system of a pendulum forms a frictionless fabric, as it optimizes
the potential function defined by gravity when being damped (i.e., it eventually comes to
rest at the configuration with minimal potential energy)

In the following, methods to construct optimization fabrics, or \textit{fabrics} for
short, are summarized: the definitions of conservative fabrics and energization are
introduced.

\subsection{Conservative fabrics and energization}%
\label{sub:conservative_fabrics_and_energization}

While the previous subsection defined what criteria are required for a spec to form
an optimization fabric, the theory on conservative fabrics and energization 
offers a simple way of generating such special specs. As a full summary of the theory
on optimization fabrics and their construction is out of scope here, this
subsection only provides an outline of the theory and the reader is referred to 
\cite{Ratliff2021,Wyk2022} for detailed derivations.

In the context of fabrics, the term \textit{energy} describes a scalar quantity that
changes as the system evolves over time.  Although this quantity has a physical meaning in
natural systems (e.g., kinetic energy), it can be arbitrarily defined for motion generation.
Generally, specs and optimization fabrics do not conserve an energy, but when they do, we
call them \textit{conservative specs}.  A stationary Lagrangian~\cite[Definition
4.11]{Ratliff2020} is one definition for an energy for which the corresponding spec, known
as the Lagrangian spec $\Spec_{\le} = \left(\Me,\fe\right)$, is obtained by applying the
Euler-Lagrange equations. Importantly, Lagrangian specs conserve energy and do thus belong
to the class of \textit{conservative specs}.  It was proven that an unbiased
(\cite[Definition~4.11]{Ratliff2020}) Lagrangian spec forms a frictionless
fabric~\cite[Proposition~4.18]{Ratliff2020}. Such fabrics are analogously called
\textit{conservative fabrics}.  There are two classes of conservative fabrics: Lagrangian
fabrics (i.e., the defining energy is a Lagrangian) and the more specific subclass of
Finsler fabrics (i.e., the
defining energy is a Finsler structure~\cite[Definition 5.4]{Ratliff2020}). 

The operation of \textit{energization} transforms a given differential equation into
a conservative spec.
Specifically, given an unbiased energy Lagrangian \le{} with boundary conforming
\Me{}~\cite[Definition~4.6]{Ratliff2020} and
lower bounded energy \he{}, an unbiased spec of form $\Spec_{\vec{h}} = (\mat{I},\vec{h})$
is transformed into a frictionless fabric using energization as
\begin{equation}
  \begin{split}
  S_{\vec{h}}^{\le} &= \text{energize}_{\le}\{S_{\vec{h}}\} \\
    &= (\Me, \fe + \Pe[\Me\vec{h} - \fe]), 
  \end{split}
\end{equation}
where $\Pe = \Me\left(\Me^{-1} - \frac{\xdot\xdot^T}{\xdot^T\Me\xdot}\right)$ is an
orthogonal projector.
Energized specs maintain the energy of the Lagrangian and generally change
the trajectory of the underlying spec $\Spec_{\vec{h}}$.
However, if 
\begin{enumerate}
  \item $\Spec_{\vec{h}} = (\mat{I},\vec{h})$ is homogeneous of degree 2,
    \begin{equation}\vec{h}(\x, \alpha\xdot) = \alpha^2\vec{h}(\x, \xdot)\label{eq:homogeneous}\end{equation}
    and
  \item the energizing Lagrangian is a Finsler structure, 
\end{enumerate}
the resulting energized spec forms a frictionless fabric for which the trajectory matches
the original trajectory of $\Spec_{\vec{h}}$. We refer to energized fabrics with that
property as \textit{geometric fabrics}. Geometric fabrics form the building blocks for
motion generation with optimization fabrics.
Practically, energization equips the individual components of the planning problem
with a metric when being combined with other components.

\subsection{Experimental results fabrics}%
\label{sub:experimental_results_fabrics}

The theory explained above was tested on several simple kinematic chains
in~\cite{Ratliff2020,Ratliff2021}. As fabrics design motion as a summation of several
differential equations, each representing a specific constraint to the motion, it is
possible to sequentially design motion~\cite{Ratliff2020}. This procedure allows to
carefully tune individual components without harming the others. The application to a
planar arm in a goal-reaching setup was successfully tested in~\cite{Ratliff2020}. Here,
the authors illustrated how the resulting motion can be modified arbitrarily by the user
by adding additional constraints or preferences.

Although important concepts and findings on optimization fabrics were summarized in this
section, we refer to~\cite{Ratliff2020} for a more in-depth presentation of optimization
fabrics. In the following, we generalize the framework of optimization fabrics to dynamic settings.

%% file: src/methods.tex
\section{Derivation of dynamic fabrics}%
\label{sec:methods}

We extend the framework of optimization fabrics to \acf{df}.
including
dynamic environments and path following tasks. We prove that \ac{df}
converge to moving goals and can be combined with previous approaches in geometric
control.
This section first introduces the notion of reference trajectory, dynamic Lagrangians and
the dynamic pullback. These notations allow then to formulate \ac{df}. 
As \ac{df} generalize the concept of optimization fabrics to dynamic scenarios, we refer 
to the non-dynamic fabrics as \acf{sf} to explicitly distinguish between the work
presented in \cite{Ratliff2020} and our work.

\subsection{Motion design using dynamic fabrics}%
\label{sub:motion_design_using_dynamic_fabrics}

The method explained in this paper generalizes the concept of \ac{sf}
from \cite{Ratliff2020} and can then be extended from the procedure outlined in \cref{sub:trajectory_generation_using_optimization_fabrics}.

\begin{enumerate}
  \item Design path-consistent geometries in a suited, \textbf{time-parameterized} (\cref{def:dynamic_map}) manifold of the configuration
  \item Design corresponding Finsler energies defining the importance metric in this manifold.
  \item \changed{Energize all geometries with the associated Finsler energies.}
  \item \textbf{If necessary, pull back the energized system from the time-parameterized manifold into
    the corresponding fixed manifold (\cref{eq:dynamic_pull})}.
  \item Pull back the energized system into the configuration space and combine it with
    all components using summation.
  \item Force the system with a \textbf{time-parameterized} potential. As a composition of \ac{df}, 
    the resulting trajectory converges towards the potential's minimum (\cref{lem:dynamic_lagrangian_fabrics}).
\end{enumerate}
In the following, we explain our proposed changes to the framework of \ac{sf} so that it remains
valid in dynamic environments.

\subsection{Reference trajectories}%
\label{sub:reference_trajectories}

To enable the definition of dynamic convergence and dynamic energy we introduce
a reference trajectory that remains inside a domain \X{} as \textit{boundary conforming}.
This term is chosen in accordance to \cite[Definition 4.6]{Ratliff2020}.
\begin{definition}
  A reference trajectory $\xt(t)$, with its corresponding time derivatives \xtdot{} and
  \xtddot{}, is boundary conforming on the manifold \X{} if $\xt(t) \in \X, \forall t$.%
  \label{def:refTraj}
\end{definition}

In the following, the reference trajectory will be used to define 
dynamic Lagrangians and dynamic fabrics. In this context, the word `dynamic' can often
be read as `relative to the reference trajectory'.
With the notion of reference trajectories we formulate a mapping to the relative
coordinate system.
\begin{definition}
  Given a reference trajectory \xt{} on \X{}, the dynamic mapping
  $\mapd : \X\times\X \to \Xr$ represents the relative coordinate system
  $\xr = \x - \xt$. 
  \label{def:dynamic_map}
\end{definition}

\begin{figure}[h]
  \centering
  \begin{subfigure}{0.5\linewidth}
    \centering
    \includegraphics[height=\textwidth]{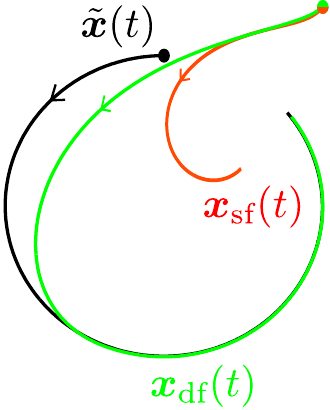}
    \caption{Dynamic convergence}%
    \label{subfig:reference_trajectory_1}
  \end{subfigure}%
  \begin{subfigure}{0.5\linewidth}
    \centering
    \includegraphics[height=\textwidth]{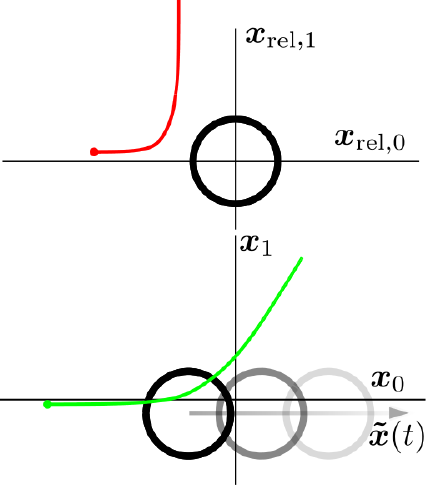}
    \caption{Dynamic Avoidance}%
    \label{subfig:reference_trajectory_2}
  \end{subfigure}
  \caption{The two implications of \acl{df}. 
    In (a), it can be seen that the trajectory obtained with \ac{df} (green)
    converges towards the reference trajectory (black) while the trajectory
    with \acl{sf} (red) does not converge. 
    In (b), the top part
    visualizes collision avoidance as suggested in~\cite{Ratliff2020}. Here,
    the trajectory and obstacle are expressed
    in a relative system $\xr$. Using the dynamic pull, \cref{eq:dynamic_pull}, this
    can be transformed into the static reference frame $\x$, bottom part. Together
    with dynamic energization, the framework of optimization fabrics is leveraged
    for dynamic environments.
    The motion of the obstacle, $\xr(t)$ is visualized with an arrow, future positions
    of the obstacle are shown in lighter color. The resulting trajectory obtained with
    \ac{df} is shown in green.
  }%
  \label{fig:reference_trajectory}
\end{figure}

\subsection{Dynamic pullback}%
\label{sub:dynamic_pullback}

The theory of optimization fabrics also applies to relative coordinates \xr{}, specifically,
specs and potentials can be formulated in moving coordinates. However, 
there is no theory to combine specs defined in relative coordinates with specs in fixed
coordinates.
In most cases, individual components of the
behavior design are not formulated in the same relative coordinates. Specificially, the
configuration space is always static, so we introduce a transformation of a relative spec
into the static space \X{}. We call this operation \textit{dynamic pullback}.
\begin{equation}
  \pull{\mapd}{{(\Md,\fd)}_{\Xr}} = {(\Md, \fd - \Md\xtddot)}_{\X}
  \label{eq:dynamic_pull}
\end{equation}
Two specs $\Spec_{\X_{\textrm{rel,1}}}$ and $\Spec_{\X_{\textrm{rel,2}}}$ defined
in two different relative coordinate systems are then combined by
first applying the dynamic pullback to both individually and then applying the summation
operation for specs. The dynamic pullback is the natural extension to optimization fabrics
for relative coordinate systems. It cannot directly be
integrated into the framework of optimization
fabrics as it breaks the algebra. In the following, we derive several
generalizations so that the theory remains 
valid even in the presence of reference trajectories for individual components, such as moving
obstacles or reference trajectories.

\subsection{Dynamic Lagrangians}%
\label{sub:dynamic_lagrangians}

Next, we show that energy conservation commutes with the dynamic pullback. This allows us to
transfer findings on conservative fabrics to dynamic fabrics.
We call a Lagrangian that is defined using relative coordinates
a \textit{dynamic Lagrangian} and write $\ld(\xr, \xrdot)$. 
In this relative coordinate system, the dynamic
Lagrangian has the same properties as the Lagrangian defined in~\cite{Ratliff2020},
specifically it induces the Lagrangian spec through the Euler-Lagrange equation, 
$\dertwo{\xrdot\xrdot}{\ld}\xrddot + \dertwo{\xrdot\xr}{\ld}\xrdot - \der{\xr}{\ld}$, 
as $(\Mde,\fde)$. The system's Hamiltonian $\hd = \der{\xrdot}{\ld}^T\xrdot - \ld$ is
conserved by the equation of motion as proven in~\cite{Ratliff2020}.

Applying the dynamic pullback to the dynamic Lagrangian we obtain the transformed
Lagrangian $\ld(\x, \xdot, \xt, \xdot)$ in the static coordinate system.

\begin{theorem}
  Let $\ld(\xr,\xrdot)$ be a dynamic Lagrangian and let \mapd{} be
  the dynamic mapping to \xr{}. Then, the application of the Euler-Lagrange equation
  commutes with the dynamic pullback.
  \label{the:dynamic_euler_lagrange}
\end{theorem}

\begin{proof}
  We will show the equivalence by calculation. As shown above, the induced spec is defined
in the relative system as $(\Mde, \fde)$. It can be dynamically pulled to form
\begin{equation}
  \pull{\mapd}{{(\Mde,\fde)}_{\Xr}} = {(\Mde, \fde - \Mde\xtddot)}_{\X}
  \label{eq:proof_theorem_euler_commutes}
\end{equation}
We can dynamically pull the Lagrangian $\ld(\xr,\xrdot)$
to form $\ld(\x,\xdot,\xt,\xtdot)$, where only the first two variables are system
variables. Using the generalized Euler-Lagrange equation, the equations of motion of the
pulled Lagrangian are obtained as
\begin{align*}
  0 &= \dert{\derf{\xdot}{\ld}}
      - \derf{\x}{\ld} \\
    &=  \derftwo{\xdot}{\xdot}{\ld}\xddot
      + \derftwo{\xdot}{\x}{\ld}\xdot
      + \derftwo{\xdot}{\xtdot}{\ld}\xtddot
      + \derftwo{\xdot}{\xt}{\ld}\xtdot
      - \derf{\x}{\ld} \\
    &=  \derftwo{\xrdot}{\xrdot}{\ld}\derf{\xdot}{\xrdot}\derf{\xdot}{\xrdot}\xddot 
      + \derftwo{\xrdot}{\xr}{\ld}\derf{\xdot}{\xrdot}\derf{\x}{\xr}\xdot \\
      &+ \derftwo{\xrdot}{\xrdot}{\ld}\derf{\xdot}{\xrdot}\derf{\xtdot}{\xrdot}\xtddot
      + \derftwo{\xrdot}{\xr}{\ld}\derf{\xdot}{\xrdot}\derf{\xt}{\xr}\xtdot \\
      &- \derf{\xr}{\ld}\derf{\x}{\xr} \\
    &=  \derftwo{\xrdot}{\xrdot}{\ld}\xddot
      + \derftwo{\xrdot}{\xr}{\ld}\xdot
      - \derftwo{\xrdot}{\xrdot}{\ld}\xtddot \\
      &- \derftwo{\xrdot}{\xr}{\ld}\xtdot
      - \derf{\xr}{\ld} \\
    &=  \derftwo{\xrdot}{\xrdot}{\ld}(\xddot - \xtddot)
      + \derftwo{\xrdot}{\xr}{\ld}(\xdot - \xtdot)
      - \derf{\xr}{\ld} \\
    &=  \derftwo{\xrdot}{\xrdot}{\ld}(\xddot - \xtddot)
      + \derftwo{\xrdot}{\xr}{\ld}(\xrdot)
      - \der{\xr}{\ld} \\
    &= \Mde \xddot + \fde - \Mde\xtddot
\end{align*}
The obtained equations of motion match the one obtained by applying the dynamic pullback,
see \cref{eq:proof_theorem_euler_commutes}.
\end{proof}
Hence, independently of the coordinates, the system conserves the energy \hd{} computed with the
Hamiltonion in relative coordinates.
Next, we adapt the operation of energization to dynamic Lagrangians. 
Dynamic Lagrangians are a necessary step to allow
for collision avoidance with dynamic obstacles in the framework of
optimization fabrics.
Specifically, the metric for a moving obstacle is computed
using the Euler-Lagrange equation in the relative coordinate system. Importantly, 
in this system, the same energies as with \ac{sf} can be employed. Using the dynamic
pullback, the energy defining the metric for the moving obstacle is then maintained
according to \cref{the:dynamic_euler_lagrange}. Concretely, this means that
collision avoidance can be achieved in a similar manner as with \ac{sf} with the added
advantage of integrated motion estimates of obstacles.

\begin{proposition}[Dynamic Energization]
  Let ${\xddot + \vec{h}(\x,\xdot) = \vec{0}}$ be a differential equation and suppose \ld{} is a
dynamic Lagrangian with the induced spec $(\Mde,\fde)$ and dynamic energy \hd{}. Then the
dynamically energized system~{$\xddot + \vec{h}(\x,\xdot) + \alpha_{\hd}\xrdot = \vec{0}$} with 
\[
  \alpha_{\hd} = -{(\xrdot^T\Mde\xrdot)}^{-1}\xrdot^T(\Mde(\vec{h}+\xtddot)-\fde)
\]
conserves the dynamic energy \hd{}.%
\label{prop:dynamic_energization}
\end{proposition}
\begin{proof}
  From the derivations in~\cite{Ratliff2020}, we can compute the rate of change of the
dynamic energy as $\dot{\hd} = \xrdot^T(\Mde\xrddot + \fde)$. The equations of motion can be
plugged in through the definition of the reference trajectory \cref{def:refTraj}, 
$\xrddot = \xddot - \xtddot$ to obtain:
\begin{align*}
  \dot{\hd} &= \xrdot^T(\Mde(-\vec{h}-\alpha_{\hd}\xrdot - \xtddot) + \fde) \\
  &= \xrdot^T(-\Mde\vec{h}-\Mde\xrdot\alpha_{\hd} - \Mde\xtddot + \fde) \\
  &= \xrdot^T(-\Mde\vec{h} \\
      & +\Mde\xrdot{(\xrdot^T\Mde\xrdot)}^{-1}\xrdot^T(\Mde(\vec{h}+\xtddot)-\fde) \\
      & - \Mde\xtddot + \fde) \\
  &= -\xrdot^T\Mde\vec{h}
      +\xrdot^T(\Mde(\vec{h}+\xtddot)-\fde) \\
      & - \xrdot^T\Mde\xtddot + \xrdot^T\fde \\
  &= 0
\end{align*}
The energized system conserves the dynamic energy.
\end{proof}
\cref{prop:dynamic_energization} allows to combine dynamic components of the
motion generator
with static components. Effectively, the dynamic component \textit{bends} the
underlying geometry according to the motion of the dynamic components (e.g., a moving
obstacle).

While dynamic Lagrangians and the corresponding energization operation are similar to the
methods described in~\cite{Ratliff2020}, the operation of the standard pull to the
dynamically energized system must be slightly modified. Specifically, the reference
velocity must be pulled. We show that dynamic
energization also commutes with the standard pullback.

\begin{theorem}
Let \ld{} be a dynamic Lagrangian to the reference trajectory \xt{}, and
let~{$\xddot+\h(\x, \xdot) = \vec{0}$} be a second order differential equation with a metric \Md{}
such that $\Jt\Md\J$ has full rank that can be written as spec $(\Md, \Md\h)$. Suppose $x =
\map(\q)$ is a differential map with \J{} its Jacobian. Then
\begin{equation}
  \energize{\pull{\map}{\ld}}{\left(\pull{\map}{(\I, \h)}\right)} =
\pull{\map}{\left(\energize{\ld}{(\I,\h)}\right)}, 
\end{equation}
when the reference velocity is being pulled as $\qtdot = \pinv{\J}\xtdot$.
$\pinv{\J}$ denotes the pseudo-inverse of \J{}.
We say that the dynamic energization operation commutes with the pullback transform.
\end{theorem}
\begin{proof}
  The commutation can be proven by calculation. First, we compute the right side of the
equivalence. According to \cref{prop:dynamic_energization}, the energized system (that
maintains the dynamic energy \hd{}) writes as 
\[
  \Md\xddot + \Md\h + \alpha_{\hd}(\xdot-\xtdot) = 0, 
\]
with $\alpha_{\hd}$ as defined in \cref{prop:dynamic_energization}.
Applying the pull-operation, we obtain
\begin{equation}
  \Jt\Md\J\qddot + \Jt\Md\h + \Jt\Md\Jdot\qdot + \Jt\Md\alpha_{\hd}(\xdot-\xtdot) = 0.
  \label{eq:pulled_energized_system}
\end{equation}
As the equation expressed in \X{}, this equation in \Q{} maintains the energy \hd{}. Next,
we compute the left hand side. The equation of motion of the pulled dynamic Lagrangian \ld{}
computes as 
\begin{align*}
  \pull{\map}{(\Md,\fd)}
    &= \Jt\left(\Md\J\qddot  + \fd + \Md\Jdot\qdot - \Md\xtddot\right)\\
    &= \tilde{\Md}\qddot  + \tilde{\fd} - \Jt\Md\xtddot.
\end{align*}
The original spec is pulled accordingly
\begin{align*}
  \pull{\map}{(\Md,\Md\h)} &=
  \Jt\Md\J\qddot + \Jt\Md\h + \Jt\Md\Jdot\qdot \\
  &= \tilde{\Md}\qddot + \tilde{\Md}\tilde{\h}
\end{align*}
We energize the pulled system according to \cref{prop:dynamic_energization}
\begin{equation}
  \begin{split}
    \Jt\Md\J\qddot + \Jt\Md\h + \Jt\Md\Jdot\qdot \\
    + \Jt\Md\J\alpha_{\pull{\map}{\hd}}(\qdot-\pinv{\J}\xdot) = 0, 
  \end{split}
\label{eq:energized_pulled_system}
\end{equation}
with 
\begin{align*}
  \alpha_{\pull{\map}{\hd}} = 
    & -{\left({(\qdot-\pinv{\J}\xtdot)}^T\Jt\Md\J(\qdot-\pinv{\J}\xtdot)\right)}^{-1} \\
    & {(\qdot-\pinv{\J}\xtdot)}^T\left(\Jt\Md\J(\tilde{\h}+\xtddot)\right. \\
    & - \left.\Jt\fd - \Jt\Md\Jdot\qdot\right) \\
    = & -{\left({(\J\qdot-\J\pinv{\J}\xtdot)}^T\Md(\J\qdot-\J\pinv{\J}\xtdot)\right)}^{-1} \\
    & {(\qdot-\pinv{\J}\xtdot)}^T\left(\Jt\Md\J\tilde{\h}+\Jt\Md\xtddot\right. \\
    & \left.- \Jt\fd - \Jt\Md\Jdot\qdot\right) \\
    = & -{\left({(\xdot-\xtdot)}^T\Md(\xdot-\xtdot)\right)}^{-1} \\
    & {(\qdot-\pinv{\J}\xtdot)}^T\left(\Jt\Md\h + \Jt\Md\Jdot\qdot +\Jt\Md\xtddot\right. \\
    & \left.- \Jt\fd - \Jt\Md\Jdot\qdot\right) \\
    = & -{\left({(\xdot-\xtdot)}^T\Md(\xdot-\xtdot)\right)}^{-1} \\
    & {(\qdot-\pinv{\J}\xtdot)}^T\left(\Jt\Md\h + \Jt\Md\xtddot- \Jt\fd \right) \\
    = & -{\left({(\xdot-\xtdot)}^T\Md(\xdot-\xtdot)\right)}^{-1} \\
    & {(\J\qdot-\J\pinv{\J}\xtdot)}^T\left(\Md\h + \Md\xtddot- \fd \right) \\
    = & -{\left({(\xdot-\xtdot)}^T\Md(\xdot-\xtdot)\right)}^{-1}
    {(\xdot-\xtdot)}^T\\
    & \left(\Md\h + \Md\xtddot- \fd \right) \\
    & = \alpha_{\hd}
\end{align*}
Thus, we have shown equivalence between $\alpha_{\hd}$
and $\alpha_{\pull{\map}{\hd}}$. As $\alpha$ is scalar we can can rewrite the energization
term in \cref{eq:energized_pulled_system} as
\begin{align*}
  &\Jt\Md\J\alpha_{\pull{\map}{\hd}}(\qdot-\pinv{\J}\xdot)\\
  = &\Jt\Md\alpha_{\pull{\map}{\hd}}(\J\qdot-\J\pinv{\J}\xdot)\\
  = &\Jt\Md\alpha_{\pull{\map}{\hd}}(\xdot-\xdot)\\
  = &\Jt\Md\alpha_{\hd}(\xdot-\xdot)\\
\end{align*}
With the equivalence of the energization terms, we conclude the
proof that dynamic energization commutes with the standard pullback.
\end{proof}

\subsection{Dynamic fabrics}%
\label{sub:dynamic_fabrics}

With the previous results, we formulate a new class of fabrics that converge to a
reference trajectory. We call this class of fabrics \acl{df}. 
First, some notations are introduced to eventually show that dynamically energized specs
form dynamic fabrics.
%
%
Analogously to unbiased specs, we define dynamically unbiased specs (i.e., specs whose
solutions do not diverge from the reference \xt{} when starting on the reference).
\begin{definition}
A spec is said to be \textit{dynamically unbiased} with respect to $\xt(t)$ if
$\f(\x, \xdot) = -\M\xtddot$, for $\x(t) = \xt(t)$ and $\xdot(t) = \xtdot(t)$.
\end{definition}

Beside being dynamically unbiased, some specs will converge to the reference trajectory
independently from their initial conditions.

\begin{definition}
A spec is \textit{dynamically rough} with respect to $\xt(t)$ if all
its integral curves $\x(t)$ converge dynamically with respect to $\xt(t)$.
\end{definition}

As for \ac{sf}, \ac{df} can be formed by
specs when they are being forced by a potential function \forc. Such a forcing potential is
generally a function of \x{} and \xt{} and has at least one minimum. A spec that converges
to a minimum of the forcing potential then forms a dynamic fabrics.

\begin{definition}
A spec forms a \textit{dynamically rough fabric} if it is dynamically rough with respect
to $\xt(t)$ when forced by a dynamic potential and 
$\exists t_1 > 0$ such that $\forall t>t_1, \x(t)$ satisfies the
Karush-Kuhn-Tucker (KKT) conditions for the optimization problem $\text{min}_{\x\in\X} \forc(\x,
\xt(t))$. If a spec does not form a dynamically rough fabric but all its damped variants
do, it forms a \textit{dynamically frictionless fabric}.
\end{definition}

\begin{theorem}[Dynamic Fabrics]
Suppose $S={(\M,\f)}_{\X}$ is a spec.  Then $S$ forms a dynamically rough fabric with
respect to \xt{} if and only if it is dynamically unbiased with respect to \xt{} and it
converges dynamically when being forced by a dynamic potential $\forc(\x,\xt)$ with
$\norm{\der{\x}{\forc}} < \infty$ on \X{}.%
\label{the:dynamic_fabrics}
\end{theorem}

\begin{proof}
We can write the corresponding differential equation
\begin{equation}
  \M\xddot + \f = -\der{\x}{\forc}
  \label{eq:proof_4_10_b}
\end{equation}
Assume that $S$ is dynamically unbiased.
Since the spec converges with respect to $\xt(t)$, we have $\xdot\to\xtdot, \x\to\xt$.
Because it is dynamically unbiased we also have $\f\to-\M\xtddot$.
Thus, the left hand side of
\cref{eq:proof_4_10_b}, approaches $\vec{0}$.
Consequently, the right hand side must also
approach $\vec{0}$ and hence $\der{\x}{\forc} \to \vec{0}$. The last satisfies the 
Karush–Kuhn–Tucker (KKT)
conditions of \forc{}.

To prove the converse, assume \f{} dynamically biased. That implies that 
\[
  \exists t > 0, \f = \M\xtddot + \vec{a}(\xt, \xtdot),
    \vec{a}(\xt, \xtdot) \neq \vec{0}.
\]
Hence, there exist a $t > 0$ for which the left hand side does not vanish. As \forc{}
satisifies the KKT conditions at $\x = \xt$, its derivative equals zero at $\x = \xt$ which contradicts
\cref{eq:proof_4_10_b} with $\M\vec{a}(\xt, \xtdot) = \vec{0}$.
\end{proof}

Hence, the spec is required to be unbiased and convergent when forced. While the former
can be verified using straight-forward computation, convergence is difficult to verify in
the general case. 

\begin{lemma}[Dynamically energized fabrics]
  Suppose $S$ is a dynamically unbiased energized spec. Then $S$
  forms a dynamically frictionless
  fabric if  $\der{\x}{\forc} = -\der{\xt}{\forc}$.%
\label{lem:dynamically_energized_fabrics}
\end{lemma}
\begin{proof}
  The equation of motion for the energized, forced and damped system writes as 
  \begin{equation}
    \xddot + \vec{h} + \alpha_{\hd}\xrdot + \mat{B}\xrdot + \der{\x}{\forc} = 0
  \end{equation}
  The systems energy (dynamic Hamiltonian) is used as a Lyapunov function to show
  convergence. The rate of change is computed as
  \begin{align*}
    \dot{\hd^{\forc}} &= \xrdot^T(
      \Mde(
        -\vec{h}
        - \alpha_{\hd}\xrdot 
        - \mat{B}\xrdot 
        - \der{\x}{\forc}
        - \xtddot) \\
      & + \fde) 
      + \dot{\forc} \\
              &= -\xrdot^T\mat{B}\xrdot
      - \xrdot^T\der{\x}{\forc}
      + \xdot^T\der{\x}{\forc}
      + \xtdot^T\der{\xt}{\forc} \\
              &= -\xrdot^T\mat{B}\xrdot \\
  \end{align*}

  As the system energy is lower bounded with $\hd + \forc \geq 0$ and 
  $\dot{\hd^{\forc}} \leq 0$, when $\mat{B}$ stricly positive definite, we must have
$\dot{\hd^{\forc}} \to 0$. Thus, \xrdot{} goes to zero. We can conclude that the system is
dynamically converging. As it it also said to be dynamically unbiased, the damped
energized system forms a dynamic fabric by \cref{the:dynamic_fabrics}.
\end{proof}

\begin{lemma}[Dynamic Lagrangian fabrics]
An unbiased, dynamic Lagrangian spec forms a dynamically frictionless
fabric if  $\der{\x}{\forc} = -\der{\xt}{\forc}$ holds for the forcing term.
\label{lem:dynamic_lagrangian_fabrics}
\end{lemma}
\begin{proof}
The equations of motion induced by the dynamic Lagrangian including damping and forcing 
are defined by the spec and can
be written explicitly as
\begin{align}
  \M_{\ld}\xrddot + \f_{\ld} + \mat{B}\xrdot + \der{\x}{\forc} = 0 \nonumber \\
  \M_{\ld}(\xddot - \xtddot) + \f_{\ld} + \mat{B}(\xdot - \xtdot) + \der{\x}{\forc} = 0
\label{eq:motion} \\
  \M_{\ld}\xddot - \M_{\ld}\xtddot + \f_{\ld} + \mat{B}\xdot - \mat{B}\xtdot +
    \der{\x}{\forc} = 0 \nonumber
\end{align}

In the following we use the Hamiltonian and the potential function as Lyapunov function to
show convergence of the damped spec.
\begin{align*}
  \hd^{\forc}(\x) &= \hd + \forc \\
  &= \der{\xrdot}{\ld^T}\xrdot - \ld + \forc
\end{align*}
The time derivative is composed of the time derivative of the Hamiltonian, $\dot{\he} =
\xrdot^T (\M_{\ld}\xrddot + \f_{\ld})$, and the time derivative of the forcing potential, 
$\dot{\forc} = \xdot^T\der{\x}{\forc} + \xtdot^T\der{\xt}{\forc}$.
Thus, the system's total energy varies over time:
\[
  \dot{\hd}^{\forc}(\x) = {(\xdot - \xtdot)}^T(\M_{\ld} (\xddot - \xtddot) + \f_{\ld}) +
\xdot^T\der{\x}{\forc} + \xtdot^T\der{\xt}{\forc}
\]
Plugging in the equations of motion \cref{eq:motion} gives
\begin{align*}
  \dot{\hd}^{\forc}(\x) &= {(\xdot - \xtdot)}^T(-\f_{\ld} - \mat{B}(\xdot - \xtdot) - \der{\x}{\forc} +
    \f_{\ld}) \\
  &  + \xdot^T\der{\x}{\forc} + \xtdot^T\der{\xt}{\forc} \\
  &= -{(\xdot - \xtdot)}^T\mat{B}(\xdot - \xtdot) - {(\xdot - \xtdot)}^T\der{\x}{\forc} \\
  & + \xdot^T\der{\x}{\forc} + \xtdot^T\der{\xt}{\forc} \\
  &= -{(\xdot - \xtdot)}^T\mat{B}(\xdot - \xtdot)
    + \xtdot^T(\der{\x}{\forc} + \der{\xt}{\forc}).
\end{align*}
For $\der{\x}{\forc} = -\der{\xt}{\forc}$ and $\mat{B}$ strictly positive definite,
$\dot{\hd}$ is strictly negative for $(\xdot - \xtdot) \neq 0$. Since $\hd^{\forc}$ is lower
bounded as composition of lower bounded function and $\dot{\hd}^{\forc} \leq 0$,
$\dot{\hd}^{\forc} \to 0$ and thus, $\xdot \to \xtdot$ and $\x \to \xt$. Hence, the spec
converges dynamically with respect to \xt{}.
As the spec is further said to be
dynamically unbiased, the damped spec forms a dynamic fabric by \cref{the:dynamic_fabrics}.
\end{proof}

Concretely, \cref{lem:dynamic_lagrangian_fabrics} allows for trajectory following
with guaranteed convergence with \ac{df}. For example, a reference trajectory 
for the robot's end-effector is defined as $\xt(t)$. Then, the potential can be designed
as $\psi = \xt(t) - \x$ (respecting the construction rule required
for \cref{lem:dynamic_lagrangian_fabrics}). In contrast to \ac{sf}, where the static
potential function is simply updated at every time step, \ac{df} makes use of the dynamics 
of the reference trajectory through the dynamic pullback.

\subsection{Construction procedure}

From the high-level procedure explained in \cref{sub:motion_design_using_dynamic_fabrics}, we can
derive the algorithm using the formal findings in this section, see \cref{alg:motion_design}.
\begin{algorithm}
  Define basic inertia as spec $\M\qddot + \M\h = 0$ \\
  \For{avoidance in avoidances}{
    Define differential map between \Q{} and $\X_i$: $\map$ or $\map_t$ \\
    Design geometry on $\X_i$: $\xddot_i + \h_{2,i} = \zerovec$ \\
    Design Finsler energy for behavior on $\X_i$: $\l_i$ \\
    Energize geometry with Finsler energy \ref{prop:dynamic_energization}\\
    \If{\map{} is time-parameterized}{
      Apply dynamic pullback to energized system\\
    }
    Apply standard pullback\\
    Add pulled avoidance component to root fabric\\
  }
  Force root system with (time-parameterized) potential\\
  \caption{Motion design with dynamic fabrics}
  \label{alg:motion_design}
\end{algorithm}

Methods to design the individual components, such as geometry and 
Finsler structures, are introduced in~\cite{Ratliff2020}.
As these design patterns do not vary for \ac{df}, they are not repeated here.
In the result section, we show some experimental examples highlighting
the comparative advantage of
optimization fabrics over model predictive schemes and the advantage of \ac{df}
over \ac{sf} for dynamic environments.

%% file: src/extension_non_holonomic.tex
\section{Extension to non-Holonomic Constraints}%
\label{sec:non_holonomic_constraints}
Mobile manipulators are often equipped with a non-holonomic base (e.g., a
differential drive mobile robot). In contrast to revolute joints for
manipulators, non-holonomic bases imply non-holonomic constraints. Based on
ideas presented in \cite{Meng2019}, we propose a method to integrate such
constraints in optimization fabrics, including \ac{df}.

We assume that the non-holonomic constraint at hand can be expressed as an
equality of form
\begin{equation}
  \xdot = \Jnh\qdot,
  \label{eq:non_holonomic_constraint}
\end{equation}
where \Jnh{} is the Jacobian of the constraint, \qdot{} is
the velocity of the controlled joints of the system and \xdot{} is the root velocity of
the fabric. For a differential drive \xdot{} is the velocity of the system in the
Cartesian plane ($\dot{x}, \dot{y}, \dot{\theta}$) and \qdot{} is the velocity of the
actuated wheels ($u_{\textrm{left}}, u_{\textrm{right}}$). Moreover, we assume that Eq.
\ref{eq:non_holonomic_constraint} is smooth and differentiable so that we can write 
\begin{equation}
  \xddot = \Jnhdot\qdot + \Jnh\qddot.
  \label{eq:non_holonomic_constraint_derived}
\end{equation}

The theory of optimization fabrics allows to pull a tree of fabrics back into one fabric
expressed in its root-coordinates of form $\M\xddot + \f = 0$ with its solution as
\begin{equation}
  \xddot = -\M^{-1}\f.
  \label{eq:fabric_solution}
\end{equation}
Plugging Eq. \ref{eq:non_holonomic_constraint_derived} into the root
fabric we obtain the non-holonomic fabric of form
\begin{align*}
  \M\Jnh\qddot + \M\Jnhdot\qdot + \f & = 0 \\
  \Mnh\qddot + \fnh & = 0 \\.
\end{align*}
Note that \Mnh{} is not necessarily a square matrix and thus not invertible as it was in the
original fabric. To find the best actuation for the wheels, we
formulate motion generation with fabrics as an
unconstrained optimization problem
\begin{equation}
  \qddot^{\ast} = \min_{\qddot}\norm{\Mnh\qddot + \fnh}_2^2.
  \label{eq:non_holonomic_fabrics}
\end{equation}
In this approach, we minimize the error of the final equation. We
could equally derive \cref{eq:non_holonomic_fabrics} with the objective
of minimizing the error between $\xddot = \Jnh\qddot+\Jnhdot\qdot$ and the
original fabric's solution $\xddot = -\M\f$. The mimization of the difference
leads to \changed{similar} result.
This optimization problem replaces \cref{eq:fabric_solution} and is solved by
\begin{equation}
  \qddot^{\ast} = \pinv{\Mnh}\fnh.
\end{equation}
The solutions to this problem makes optimization fabrics, and thus dynamic fabrics,
applicable to non-holonomic robots, such as differential drive robots or cars.
A qualitative comparison between a trajectory generated for a holonomic
and a non-holonomic robot is shown in \cref{fig:non_holonomic_trajectory}.
\begin{figure}
  \centering
  \begin{subfigure}{0.3\linewidth}
    \centering
    \includegraphics[width=0.95\textwidth]{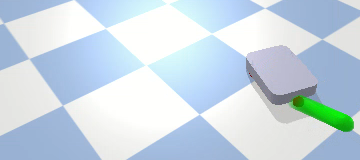}
    \caption{7s}
  \end{subfigure}%
  \begin{subfigure}{0.3\linewidth}
    \centering
    \includegraphics[width=0.95\textwidth]{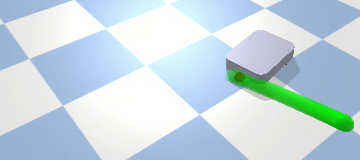}
    \caption{11s}
  \end{subfigure}%
  \begin{subfigure}{0.3\linewidth}
    \centering
    \includegraphics[width=0.95\textwidth]{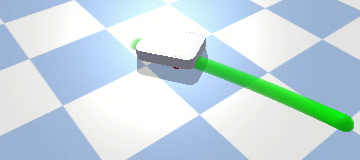}
    \caption{18s}
  \end{subfigure}
  \begin{subfigure}{0.3\linewidth}
    \centering
    \includegraphics[width=0.95\textwidth]{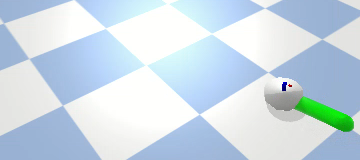}
    \caption{7s}
  \end{subfigure}%
  \begin{subfigure}{0.3\linewidth}
    \centering
    \includegraphics[width=0.95\textwidth]{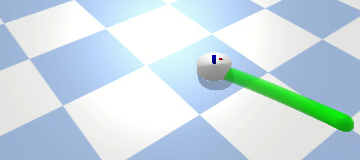}
    \caption{11s}
  \end{subfigure}%
  \begin{subfigure}{0.3\linewidth}
    \centering
    \includegraphics[width=0.95\textwidth]{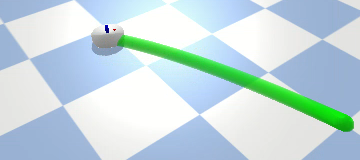}
    \caption{18s}
  \end{subfigure}%
  \caption{Path (green) following with a holonomic and a non-holonomic robot using \ac{df} with the extension
  to non-holonomic robots}
  \label{fig:non_holonomic_trajectory}
\end{figure}


The theory of optimization fabrics is built upon energy conservation of artificial energies
that design the motion. \cref{eq:non_holonomic_fabrics} does not solve the
resulting spec exactly, but minimizes the deviation according to the least square objective function.
For many kinematic systems, e.g., differential drive model, bicycle model, the
non-holonomic constraint additionally reduces the number of degrees of freedom, 
$\dim{q} < \dim{x}$. As a consequence, the least square solution has a non-zero residuum.
Then, some fundamental properties of optimization fabrics, such as energy
conservation and convergence can no longer be guaranteed. Despite this theoretical
shortcoming, we show that this approach leads to good performance in practical applications.

%% file: src/results.tex
\section{Experimental results}%
\label{sec:experimental_results}
\begin{figure*}[ht]
  \centering
  \begin{subfigure}{0.33\linewidth}
    \centering
    \includegraphics[width=0.9\textwidth]{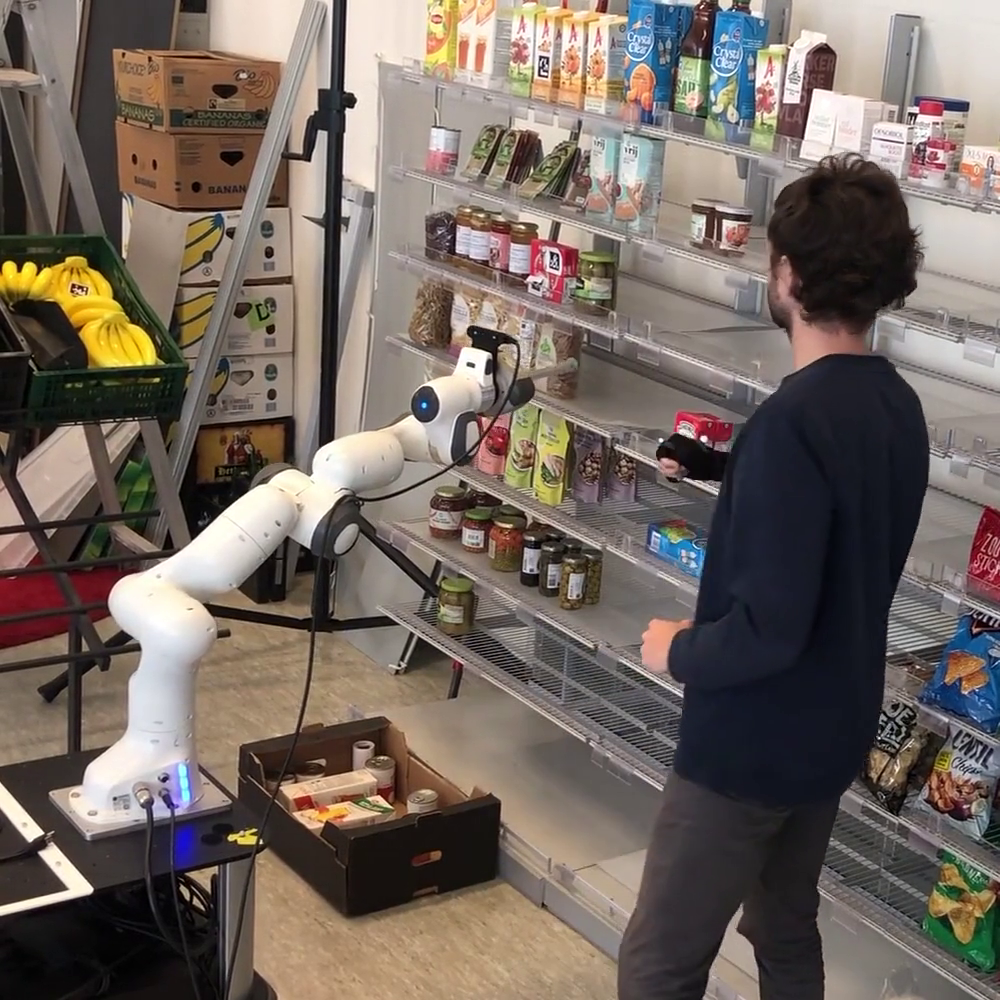}
    \caption{$t=0$s}%
    \label{subfig:experiment3_realPanda_example_1}
  \end{subfigure}%
  \begin{subfigure}{0.33\linewidth}
    \centering
    \includegraphics[width=0.9\textwidth]{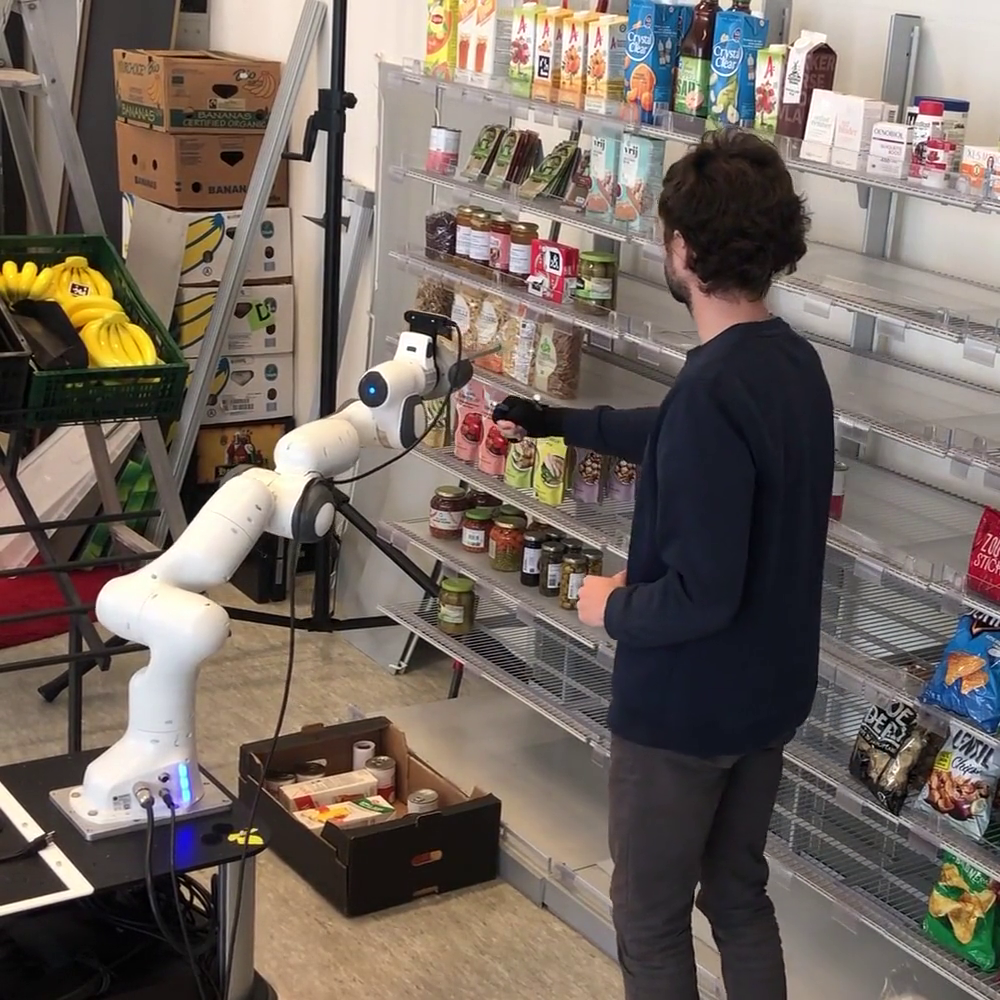}
    \caption{$t=2$s}%
    \label{subfig:experiment3_realPanda_example_1}
  \end{subfigure}%
  \begin{subfigure}{0.33\linewidth}
    \centering
    \includegraphics[width=0.9\textwidth]{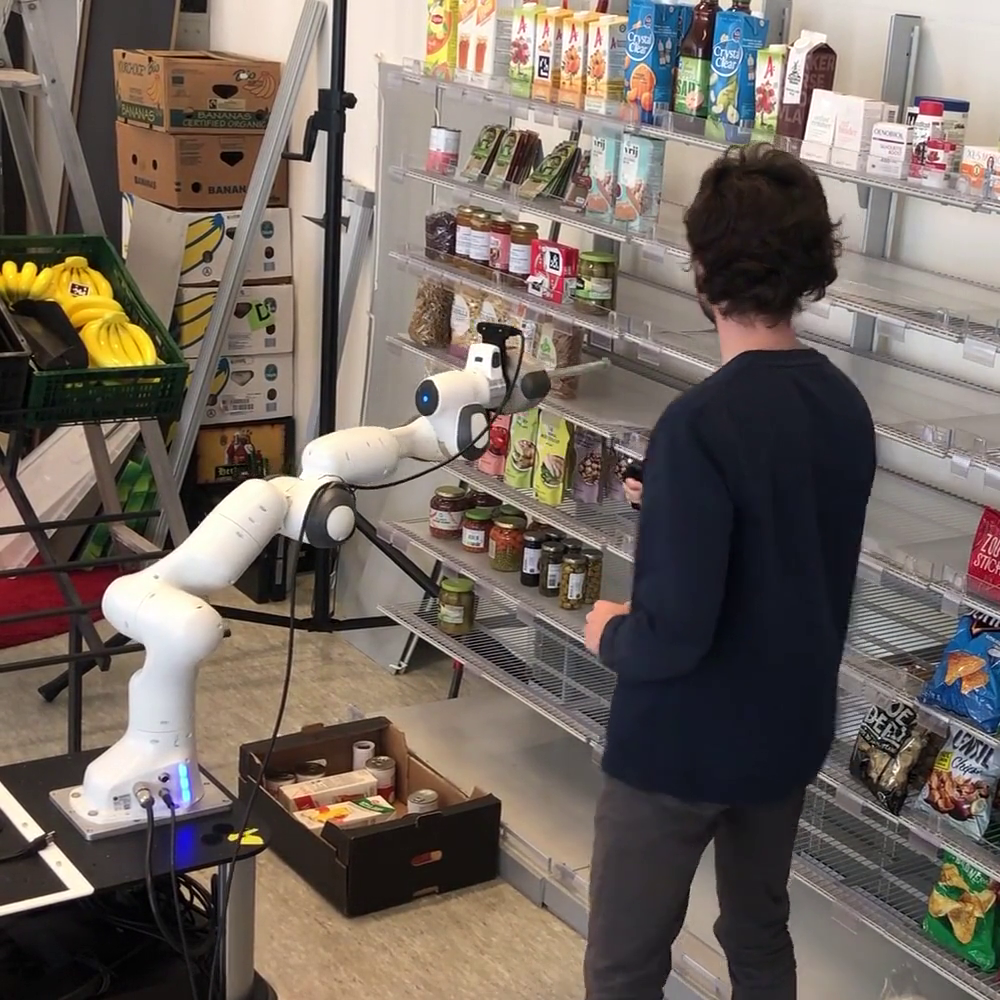}
    \caption{$t=3$s}
    \label{subfig:experiment3_realPanda_example_1}
  \end{subfigure}%
  \caption{\acl{df} in the presence of a human. The human hand's state is estimated with a motion capture
    system. The robot smoothly and in advance avoids the human operator and allows for safe coexistence.
  }%
  \label{fig:experiment6_realPanda}
\end{figure*}

In this section, the performance of optimization fabrics is assessed on various
robotic platforms. Although~\cite{Ratliff2020} suggested performance benefits
over optimization-based methods to local motion planning, no quantitative
comparisons have been presented to this date. 
The scenarios that we have chosen here (especially in the 
first two experiments) are intentionally simple to identify the specific
differences. In the real world experiments, we show the differences
on more dynamic scenarios, where the limited frequency of a global
planning method, such as RRT, justifies the need for a local planning method.
To give a general idea of the
performance differences between \ac{sf} and receding-horizon trajectory
optimization, we compare the performance of an \ac{mpc} formulation, adapted
from~\cite{Spahn2021}, with \ac{sf}, as proposed in~\cite{Ratliff2020}.  The
second experiment compares performance between \ac{sf} and \ac{df} for
trajectory following tasks. In the third experiment, moving obstacles are added
to the scene to form a dynamic environment. Our extension to non-holonomic
systems is tested in the fourth experiment. Then, everything is combined in
an experiment with a differential drive mobile manipulator.
Finally, we present a possible application of a robot sharing the environment
with a human.
The experiments described here are supported by videos accompanying this paper.

\subsection{Settings \& performance metrics}%
\label{sub:settings}

We present a detailed analysis of the experimental results for two commonly used setups,
namely the \textit{Franka Emika Panda}, a \textit{Clearpath Boxer}, and a mobile manipulator
composed of both components
see~\cite{Spahn2021}. Note, that these
robots are representative of commonly used robots in dynamic environments. The
Franka Emika Panda is a 7 degree-of-freedom robot with joint torque sensors, comparable to the
Kuka Iiwa. Mobile manipulators equipped with differential drives are widely used by other
manufacturers, see Pal Robotics Tiago or the Fetch Robotics Mobile Manipulator.

Compared to~\cite{Cheng2018}, we propose a more extensive list of metrics.
With regards to safety, we measure the \textbf{Clearance}, the minimum distance
between the robot and any obstacle along the path.
For static goals, solver planner performance is measured in terms of
\textbf{Path Length}, euclidean length of the end-effector trajectory, and
\textbf{Time-to-Goal}, time to reach the goal. For trajectory following tasks,
this measure\ is replaced by \textbf{Summed Error}, the normed sum of deviation
from the desired trajectory. Computational costs are measured by the average
\textbf{Solver Time} in each time step. Most important, binary success is
measured by the \textbf{Success Rate}, where failure indicates that either the
goal was not reached or a collision occurred during execution. Performance
metrics are only evaluated if the concerned motion generator succeeded.
More information on the testbed can be found in \cite{spahn2022local}.
 
In static, industrial environments the time to reach the goal can be considered
the one single most important metric, but we argue that dynamic environments
require a more nuanced performance evaluation and thus a set of metrics.
Intentionally, we do not give general weights to the individual metrics, as
their corresponding importance highly depends on the application. As a
consequence, we tuned the compared planners in such a way that they reach the
goal in a similar time. Note that the general speed for all planners compared
in this article can be adjusted by choosing a different parameter setup.

As this work does not focus on obstacle detection, we simplify
obstacles to spheres. Thus, we assume that an operational perception
pipeline detects obstacles and constructs englobing spheres.
The experiments are randomized in either the location of the obstacles, the location of the
goal, the initial configuration, or in a combination of all three aspects.
For every experiment, the type and level of randomization are stated.

\input{src/results_exp1}
\input{src/results_exp2}

\input{src/results_exp3}

\input{src/results_exp4}

\input{src/results_exp5}
\input{src/results_exp6}

%% file: src/results_exp1.tex
\subsection{Experiment 1: Static fabrics vs. \ac{mpc}}%
\label{sub:experiment_1_static_fabrics_vs_mpc}

In the first experiment, we compare the performance of an \ac{mpc}
formulation with \ac{sf}
\cite{Ratliff2020,Wyk2022}.
Compared to the formulation used in \cite{Spahn2021}, we use a workspace goal
rather than a configuration space goal, and apply a second order integration
scheme so that the control outputs are accelerations instead of velocities. We
clarify that the formulation deployed for the following tests is geometric as
the model used is a second order integrator and does not include the dynamics
of the robots. The main reason lies in the reduced computational costs and the
inaccessibility of a highly accurate model \cite{meo2021adaptation}.

\paragraph{Parameters}
The low-level controller of the robot runs at $1$ kHz. The fabrics are running at $100$ Hz
and the \ac{mpc} at $10$ Hz. The time horizon for the \ac{mpc}
planner was set to $T=3$s spread equally over
$H=30$ stages. Based on the findings in~\cite{Spahn2021}, we are confident that the \ac{mpc}
planner is close to its optimal settings. Moreover, we used the implementations
by~\cite{forcespro,forcesnlp}, which are reported to have improved performance over
open-source libraries like \textit{acado}.

\begin{figure}[h]
  \centering
  \begin{subfigure}{0.4\linewidth}
    \centering
    \includegraphics[height=4cm]{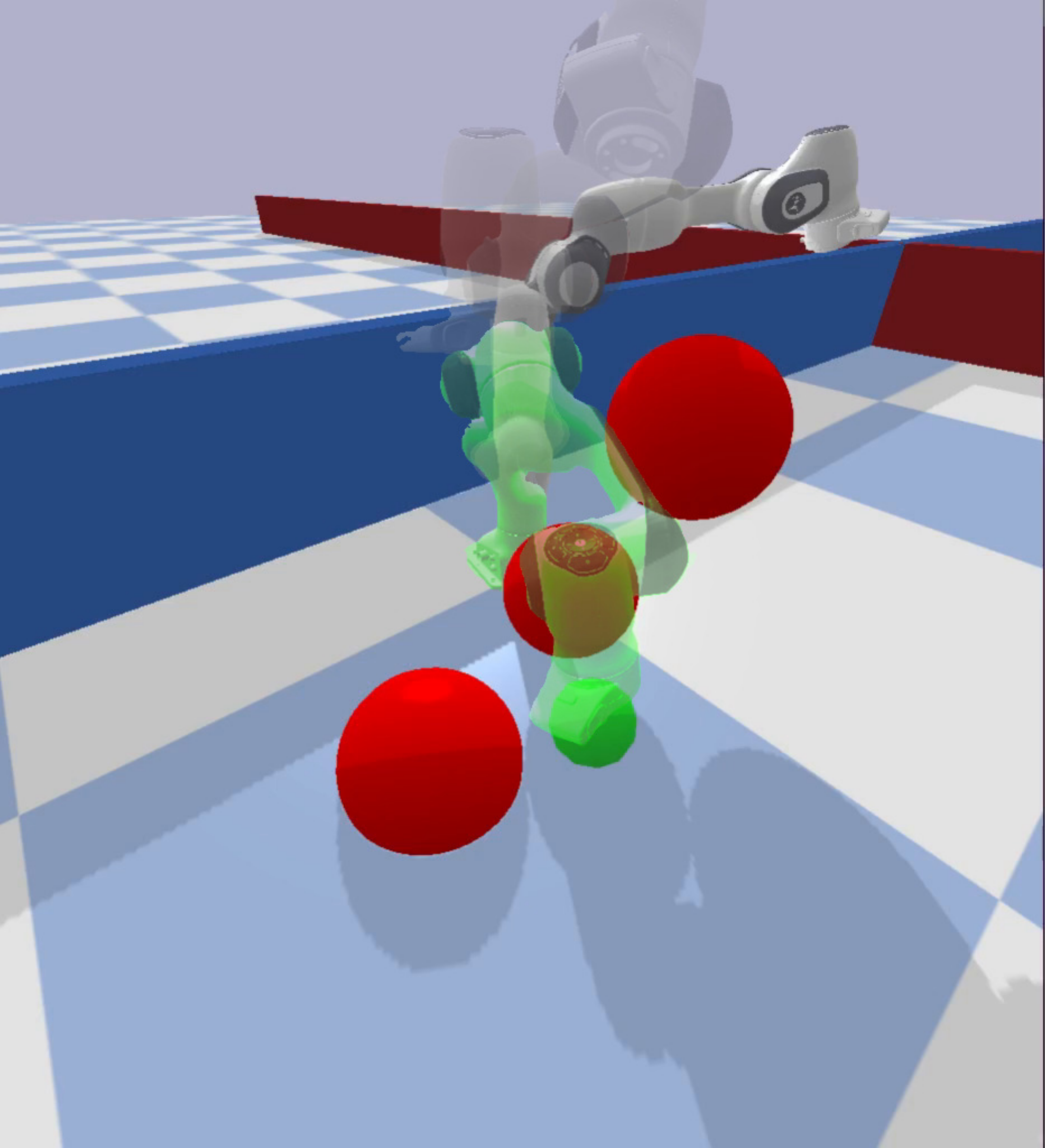}
    \caption{}%
    \label{subfig:experiment1_example}
  \end{subfigure}%
  \begin{subfigure}{0.6\linewidth}
    \centering
    \includegraphics[height=4cm]{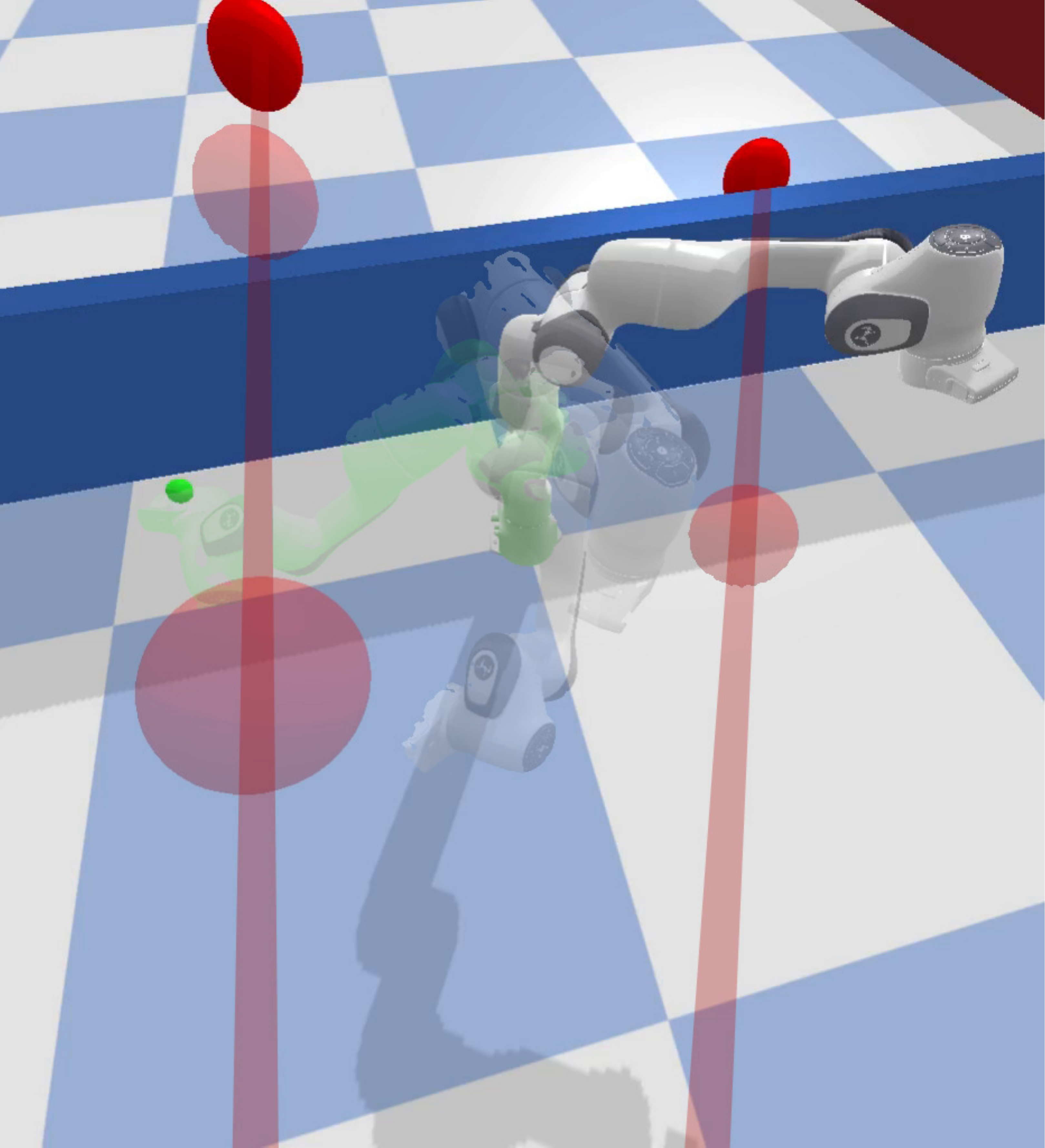}
    \caption{}%
    \label{subfig:experiment3_example}
  \end{subfigure}
  \caption{Examples for simulation setups with \panda{}.
    Initial configuration are shown in white and final configurations in light green.
    Obstacles are visualized in red. In (a), only static obstacles are considered. In (b),
    the trajectories of two moving obstacles are visualized in light red.
  }%
  \label{fig:experiments_example}
\end{figure}

\paragraph{Simulation}
A series of $N=50$ runs was evaluated with the \panda{} in simulation. Randomized
end-effector positions were set for every run, while the initial configuration
remained unchanged. One to five spherical obstacles of
radius $r=0.15$ m were
placed in the workspace at random. An example setup is shown in
\cref{subfig:experiment1_example}. The results are summarized in 
\cref{fig:experiment1_simPanda}. Solver times with fabrics averaged at $1$ ms while the
\ac{mpc} solver took around $50$ ms in every time step. Although the path length is similar
with both solvers, the minimum clearance from obstacles is increased with \ac{sf}
($0.183$ m) compared to \ac{mpc} ($0.138$ m). This means that the trajectories are
safer and thus more suitable for dynamic environments.
Both motion generation methods fail in 6 cases. However, the \ac{sf} produce only
one collision while \ac{mpc} creates 5 collisions. The remaining failures are deadlocks.
For both methods, deadlocks result from local minima, highlighting the 
need for supportive global plans. Collisions are caused
by numerical inaccuracies, which are generally higher with \ac{mpc} due
to the lower frequency.

\begin{figure}[h]
  \centering
  \begin{subfigure}{1.0\linewidth}
    \centering
    \includegraphics[angle=-90,width=\textwidth]{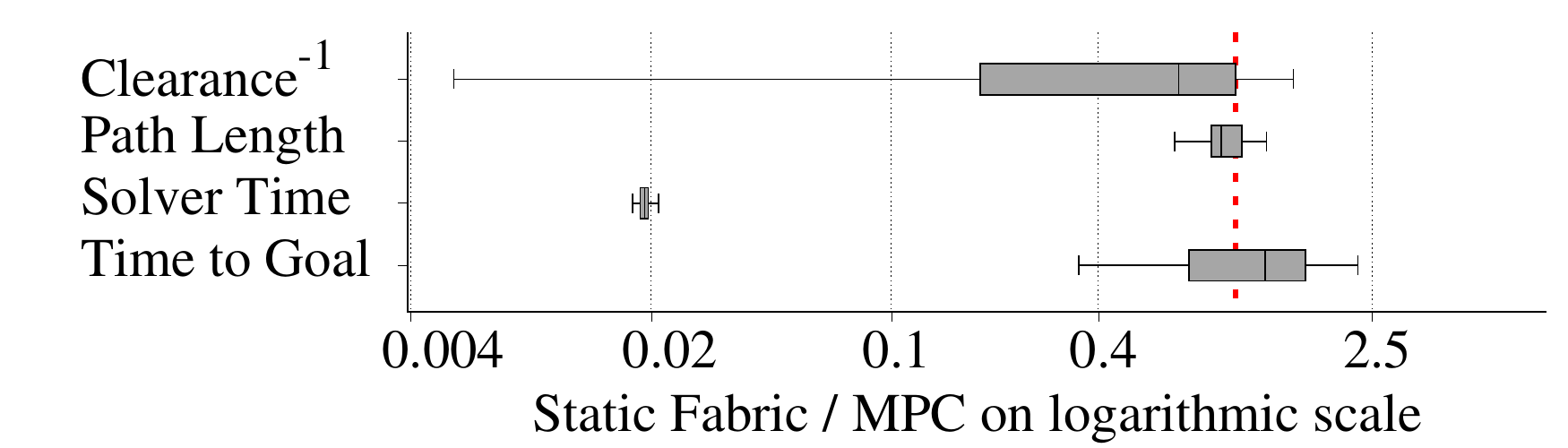}
    \caption{Metrics evaluation for successful experiments}%
    \label{subfig:experiment1_simPanda_res}
  \end{subfigure}
  \begin{subfigure}{1.0\linewidth}
    \centering
    \includegraphics[angle=-90,width=\textwidth]{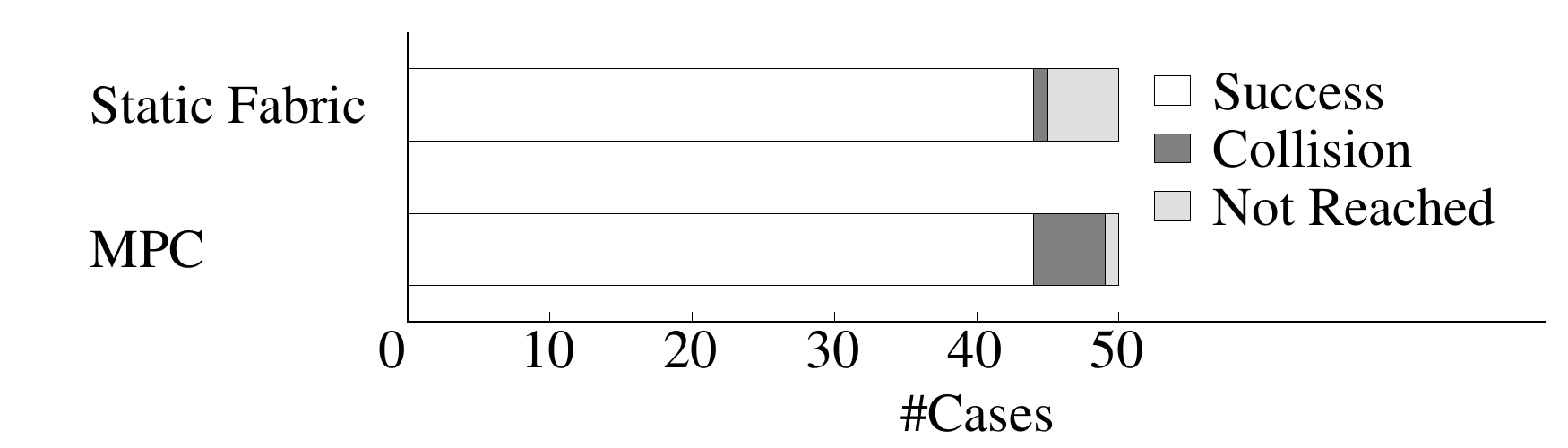}
    \caption{Success results}%
    \label{subfig:experiment1_simPanda_success}
  \end{subfigure}
  \caption{Results for randomized motion planning problems with the \panda{} in simulation.
    Lower values represent an improved performance of \ac{sf} over MPC.
  }%
  \label{fig:experiment1_simPanda}
\end{figure}

\paragraph{Real World}
For the experiments with the real robot, we limited the number of test runs to $N=20$. In
contrast to the simulated results, \ac{mpc} has significantly more collisions than
\ac{sf}, \cref{subfig:experiment1_realPanda_success}.  This is likely to be caused by the
lower frequency at which the \ac{mpc} is running. While in simulation the model matches
the actual behaviour perfectly and the time interval between two computations can be
accuratly predicted, more uncertainty in the model is present in the real world. This
leads to prediction errors that cause collisions.  For the collision free cases, the real
world experiments confirm that optimization fabrics tend to be more conservative with
respect to obstacles, see \textit{Clearance} in \cref{subfig:experiment1_realPanda_res}.
Similar to the simulated results, the solving time is reduced by a factor of around $50$
with fabrics. This
allows to run the planner at a higher frequency and thus generating smoother motions.

\begin{figure}[h]
  \centering
  \begin{subfigure}{\linewidth}
    \centering
    \includegraphics[angle=-90,width=\textwidth]{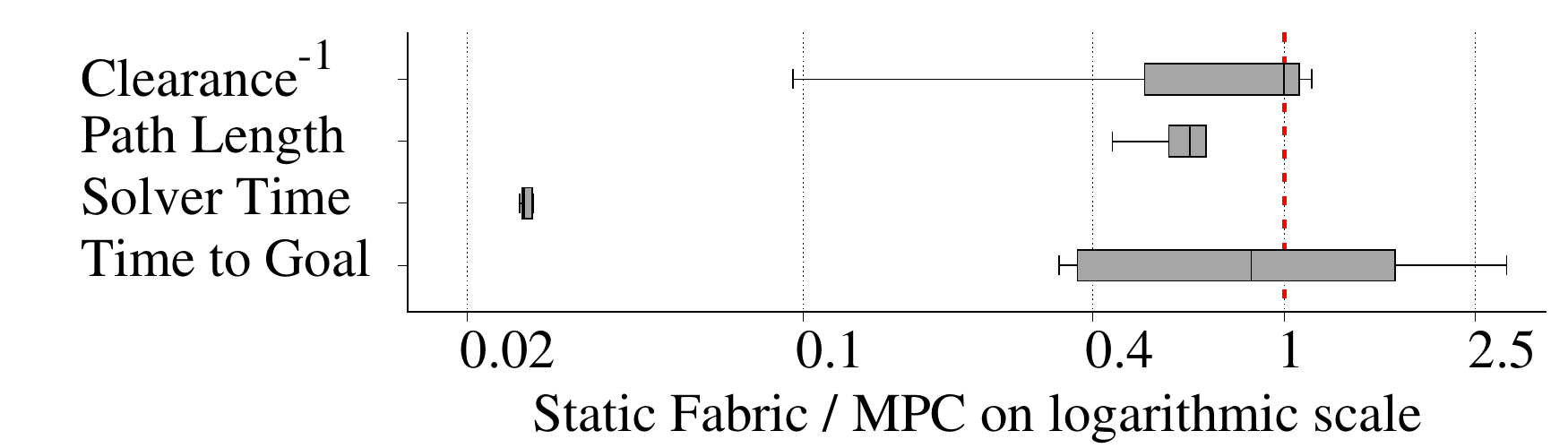}
    \caption{Metrics evaluation for successful experiments}%
    \label{subfig:experiment1_realPanda_res}
  \end{subfigure}
  \begin{subfigure}{\linewidth}
    \centering
    \includegraphics[angle=-90,width=\textwidth]{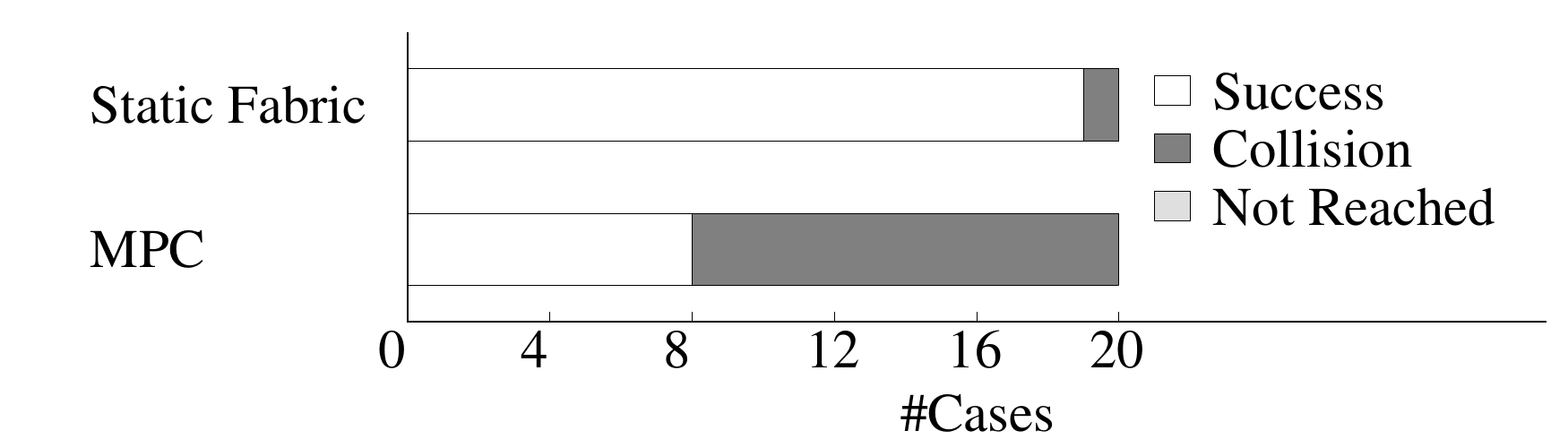}
    \caption{Success results}%
    \label{subfig:experiment1_realPanda_success}
  \end{subfigure}
  \caption{Results for randomized motion planning problems with the real \panda{}.
    \ac{sf} are more conservative around obstacles, improving on safety, while
    reducing the computational cost by a factor of $\approx$ 50.
    As a result of the increased clearance, collisions are more reliably avoided with
    \ac{sf}.
  }%
  \label{fig:experiment1_realPanda}
\end{figure}

\paragraph{Discussion}
The difference in performance (except for solver time) is likely caused by
the different objective metrics. The objective function in the \ac{mpc}
formulation is mainly governed by the Euclidean distance to the goal while
control inputs and velocity magnitude are given a relative small weight.
Avoidance behaviors, such as joint limit avoidance and obstacle avoidance,
are respected through inequality constraints. In contrast, \ac{sf} design the
objective in a purely geometric manner including all avoidance behaviors.
Thus the manifold for the motion is directly altered by the avoidance
behaviors, i.e., the manifold is \textit{bent} \cite{Ratliff2020} so that
the notion of shortest path changes with the addition of obstacles. This
shaping of the manifold leads to improved canvergence compared to the
combination of Euclidean distance objective function and inequality
constraints used with \ac{mpc}.

%% file: src/results_exp2.tex
\subsection{Experiment 2: Static fabrics vs. Dynamic fabrics}%
\label{sub:experiment_2_static_fabrics_vs_dynamic_fabrics}

In motion planning for dynamic environments, global and local planning methods work together
to achieve efficient and safe motion of the robot. However, \ac{sf}
are not designed to follow global paths. Path
following can only be achieved using a pseudo-dynamic approach where the forcing potential
is shifted in every time step without considering the dynamics of the trajectory.
Therefore, we
propose \ac{df} to allow smoother path following tasks, where the speed of
the goal is also considered during execution. 
In this second experiment, we investigate
how \ac{df} compare to \ac{sf} for path following tasks.
Specifically, we show that \ac{df} outperform \ac{sf} when
following a path generated by a global planner.

\paragraph{Simulation}
We evaluated \ac{df} on the \panda{} robot in simulation with an
analytic, time-parameterized curve and a path generated by a global planner, namely RRT
(\cref{fig:experiment2_simPanda_spline_example}).
In the case of the analytic trajectory, the three obstacles were
randomized across all runs. For the experiment with the global planner, the goal position and 
the obstacles were randomized across all runs.
A total of $N=50$ experiments were executed for this
experiment. The reduced summed error for dynamic fabrics verifies the
theoretical finding that dynamic fabrics can follow paths more closely.
The average error over all runs with the
analytic trajectory is $0.0792$m (\ac{df})
and $0.136$m (\ac{sf}), see
\cref{subfig:experiment2_simPanda_res_analytic} for the comparison.
For the spline path generated with RRT, the
average error over all runs is $0.145$m (\ac{df}) and $0.240$m (\ac{sf}), see
\cref{subfig:experiment2_simPanda_res_spline} for the comparison.

\begin{figure}[ht]
  \begin{subfigure}{0.5\linewidth}
    \centering
    \includegraphics[width=0.95\textwidth]{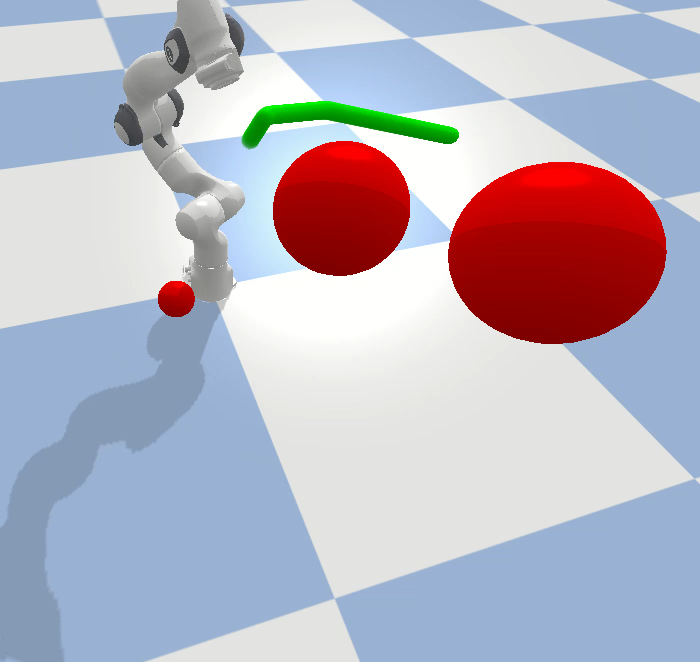}
  \end{subfigure}%
  \begin{subfigure}{0.5\linewidth}
    \centering
    \includegraphics[width=0.95\textwidth]{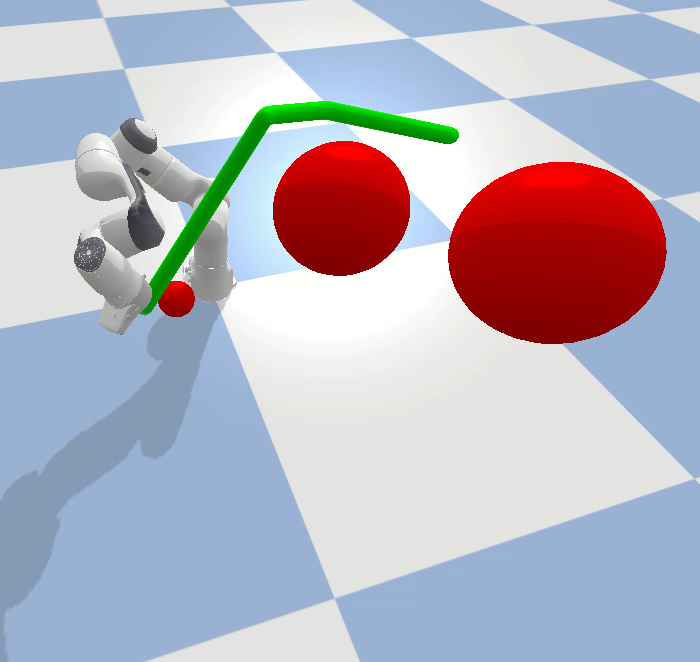}
  \end{subfigure}
  \caption{Path generated with RRT from OMPL.}
  \label{fig:experiment2_simPanda_spline_example}
\end{figure}

\begin{figure}[ht]
  \begin{subfigure}{0.5\linewidth}
    \centering
    \includegraphics[width=0.95\textwidth]{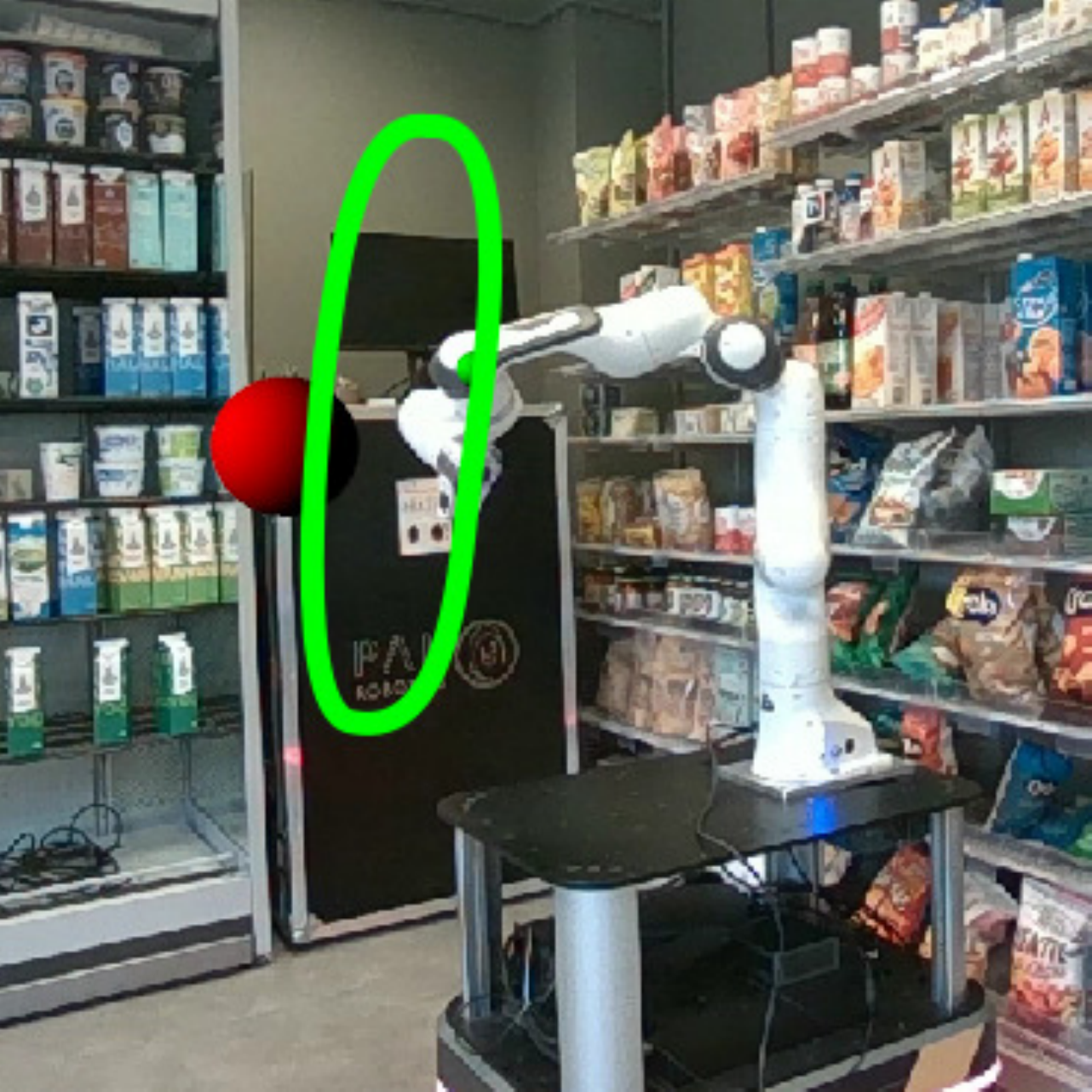}
    \caption{}
  \end{subfigure}%
  \begin{subfigure}{0.5\linewidth}
    \centering
    \includegraphics[width=0.95\textwidth]{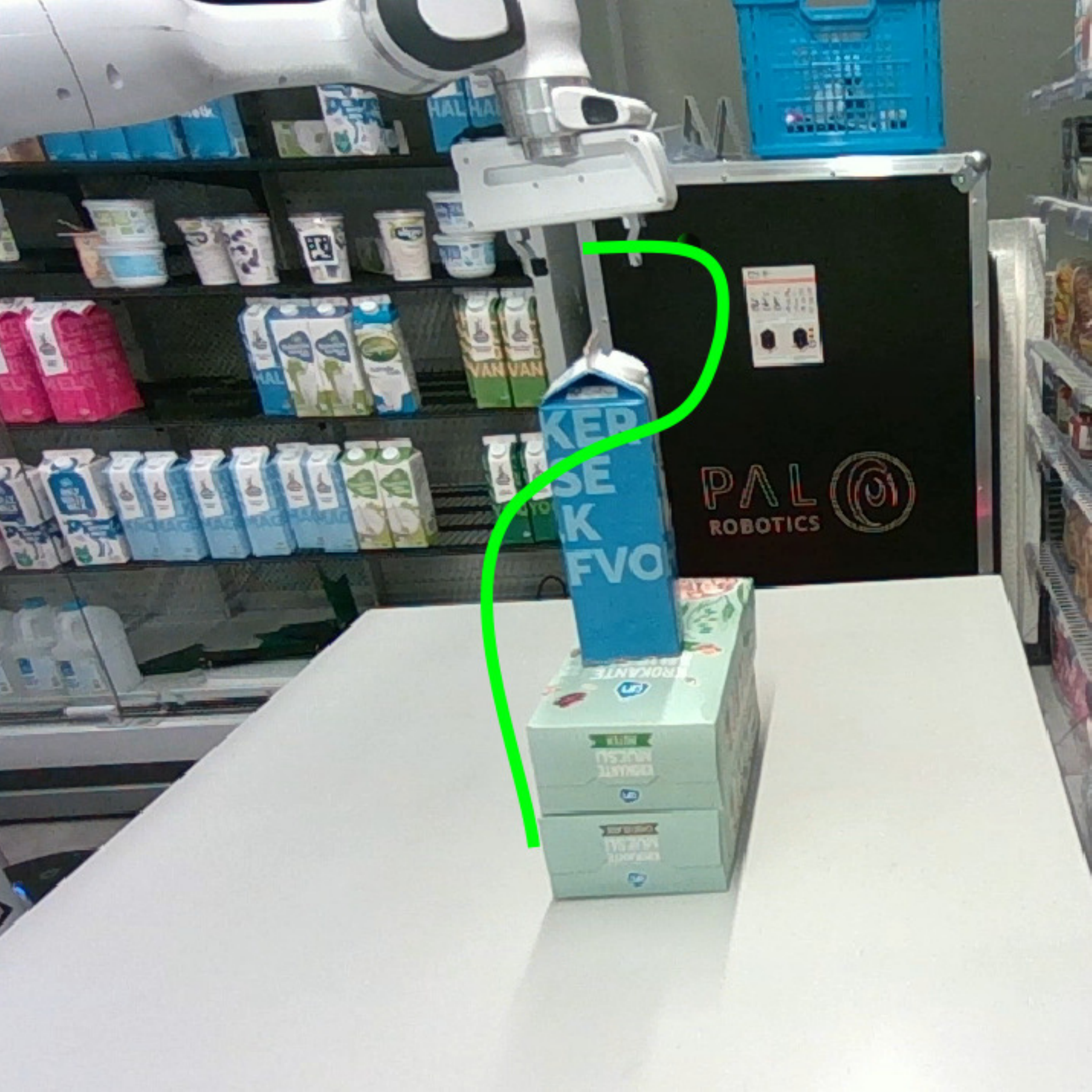}
    \caption{}
  \end{subfigure}
  \caption{Trajectory following tasks with \acl{df}. In (a), the trajectory is a 
  time-parameterized analytic curve. In (b), the trajectory is described by a spline.}%
  \label{fig:experiment2_realPanda_examples}
\end{figure}

\begin{figure}[ht]
  \centering
  \begin{subfigure}{1.0\linewidth}
    \centering
    \includegraphics[angle=-90,width=\textwidth]{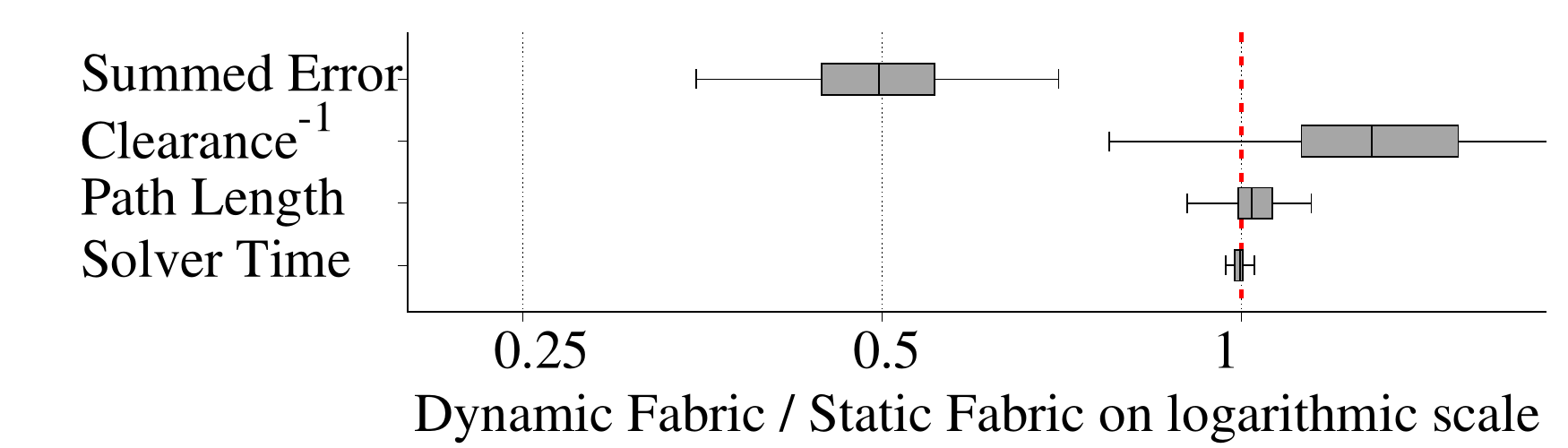}
    \caption{Analytic, user-specified global path}%
    \label{subfig:experiment2_simPanda_res_analytic}
  \end{subfigure}
  \begin{subfigure}{1.0\linewidth}
    \centering
    \includegraphics[angle=-90,width=\textwidth]{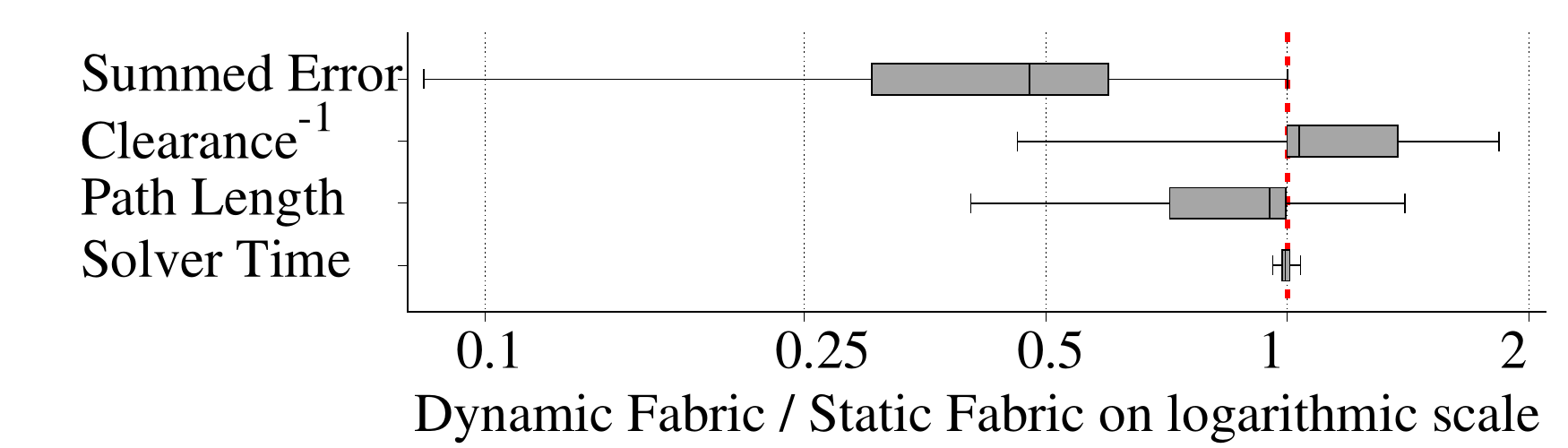}
    \caption{Global path generated by RRT using OMPL}%
    \label{subfig:experiment2_simPanda_res_spline}
  \end{subfigure}
  \caption{Comparison between static and dynamic fabrics for trajectory following tasks
    in simulation. Lower values in a metric indicate that \ac{df} performed better than
    \ac{sf}.
  }%
  \label{fig:experiment2_simPanda}
\end{figure}

\paragraph{Real-World}
Path following was also assessed with the real \panda{} in similar settings.
Quantitative results are only presented for $N=20$ different paths with splines
where up to three obstacles were added to the workspace, see \cref{fig:experiment2_realPanda_examples}. The results in real-world confirm the 
findings from the simulation. By exploiting the velocity information of the 
trajectory, the integration error can be effectively
reduced, \cref{fig:experiment2_realPanda_res}. In contrast to the simulation
we see a higher fluctuation in solver times, which can be caused by a generally lower capacity 
of the computing unit on the robot.

\begin{figure}[ht]
    \centering
    \includegraphics[angle=-90,width=\linewidth]{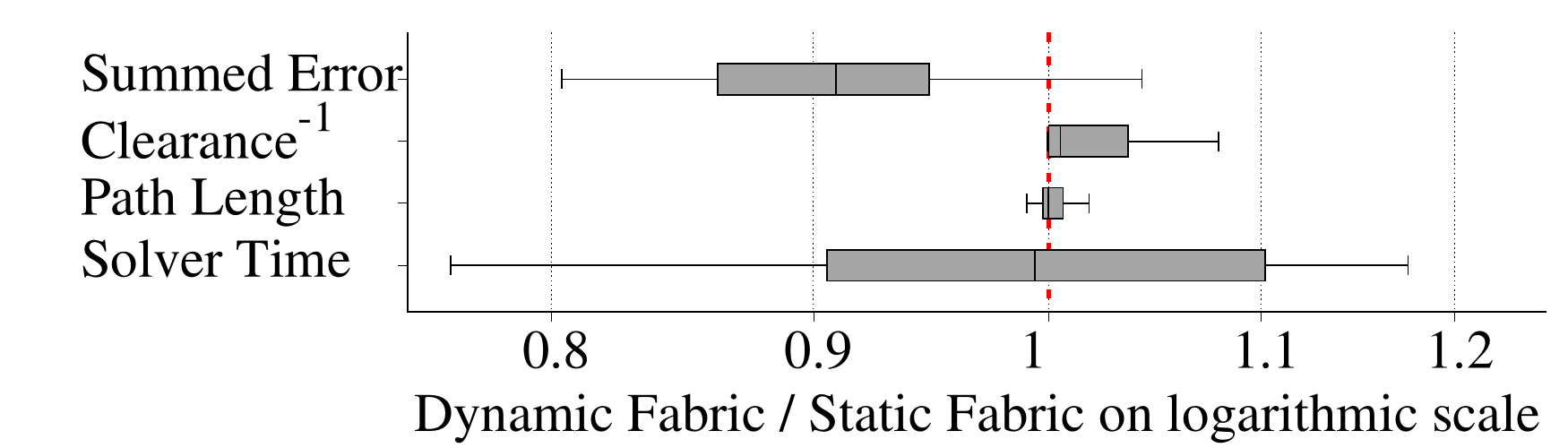}
    \caption{Comparison between \ac{sf} and \ac{df} when following a path defined
      by a basic spline in the real world. 
      The splines and the obstacles are different for the $N=20$ case.
      \ac{df} achieve lower deviation errors that \ac{sf}.
    }%
    \label{fig:experiment2_realPanda_res}
\end{figure}

%% file: src/results_exp3.tex
\subsection{Experiment 3: Moving Obstacles}%
\label{sub:experiment_3_moving_obstacles}
Next, we compare the different methods in the presence of dynamic obstacles. All
experiments in this section consist of at least one moving obstacle that follows either an
analytic trajectory or a spline. Here, we use stationary goals to isolate the results from
the behavior investigated in the previous section.

\paragraph{Simulation}
For this series with the simulated \panda{}, 
only the goal position was randomized. The initial configuration 
\[
  \q_0 = {[1.0, 0.0, 0.0, -1.5, 0.0, 1.8675]}^T, 
\]
and the two moving obstacles with the trajectories
\begin{align*}
  \xt_{\textrm{obst1}} &= {[-1.0 + 0.1t, -0.4, 0.7]}^T,  \\
  \xt_{\textrm{obst2}} &= {[-1.0 + 0.2t, 1.0 - 0.1t, 0.3]}^T
\end{align*}
were kept constant throughout all experiments. The environment is visualized in
\cref{subfig:experiment3_example}. The comparison between \ac{sf} and \ac{df}
shows that \ac{df} are more conservative in terms of collision avoidance with
dynamic obstacles. Specifically, the distance between the robot and the
obstacles is increased (\cref{subfig:experiment3_simPanda_res}).
The success rate with \ac{df} compared to \ac{sf} is significantly improved,
see \cref{subfig:experiment3_simPanda_success}. Thus showing the need for
using \ac{df} in dynamic environments.
\begin{figure}[h]
  \centering
  \begin{subfigure}{1.0\linewidth}
    \centering
    \includegraphics[angle=-90,width=\textwidth]{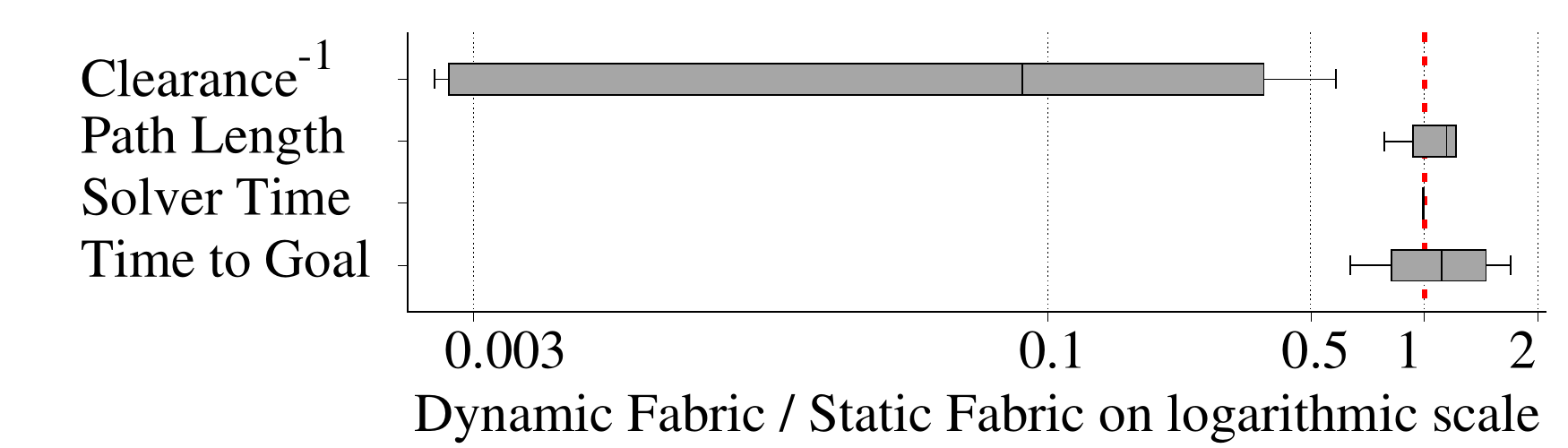}
    \caption{Metrics evaluation for sucessful experiments}%
    \label{subfig:experiment3_simPanda_res}
  \end{subfigure}
  \begin{subfigure}{1.0\linewidth}
    \centering
    \includegraphics[angle=-90,width=\textwidth]{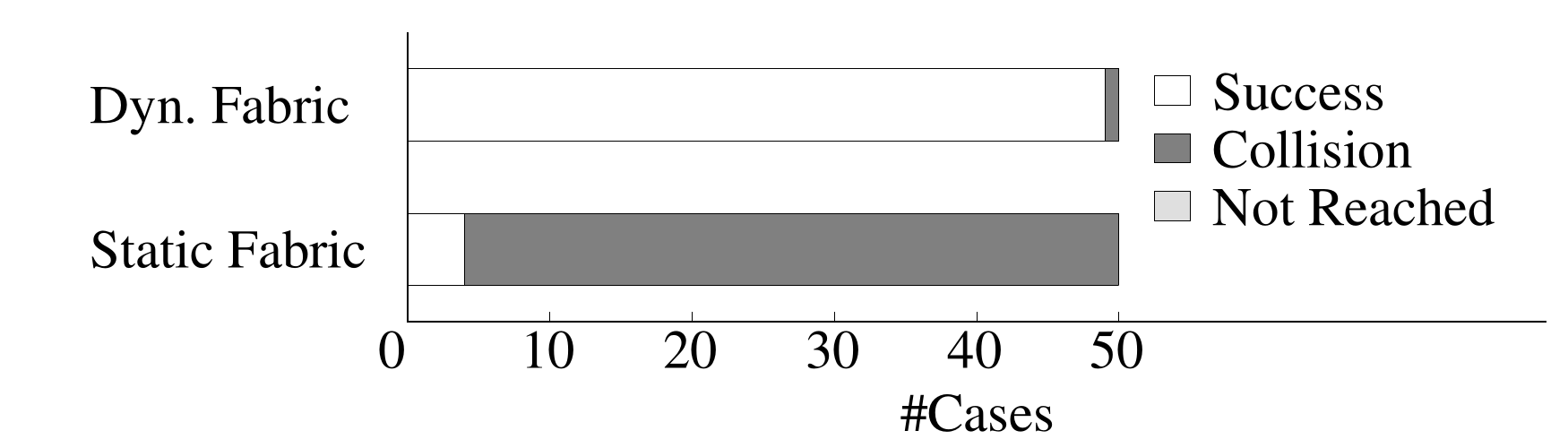}
    \caption{Success results}%
    \label{subfig:experiment3_simPanda_success}
  \end{subfigure}
  \caption{Comparison between \ac{sf} and \ac{df} for scenarios with dynamic obstacles.
    While path length and solver time is not increased, clearance is increased and
    the time to reach the goal is reduced with \ac{df} compared to \ac{sf}.
  }%
  \label{fig:experiment3_simPanda}
\end{figure}
\paragraph{Real-World}
In a series of $N=20$ experiments, performance on the real panda arm was assessed. The same
trend for more conservative behavior with \ac{df} compared to \ac{sf} can be observed,
\cref{fig:experiment3_realPanda}. However, \ac{df} take longer on average to reach the
goal as they keep larger clearance from obstacles. Note that collisions are effectively
eliminated with \ac{df} compared to \ac{sf}.

\begin{figure}[h]
  \centering
  \begin{subfigure}{1.0\linewidth}
    \centering
    \includegraphics[angle=-90,width=\textwidth]{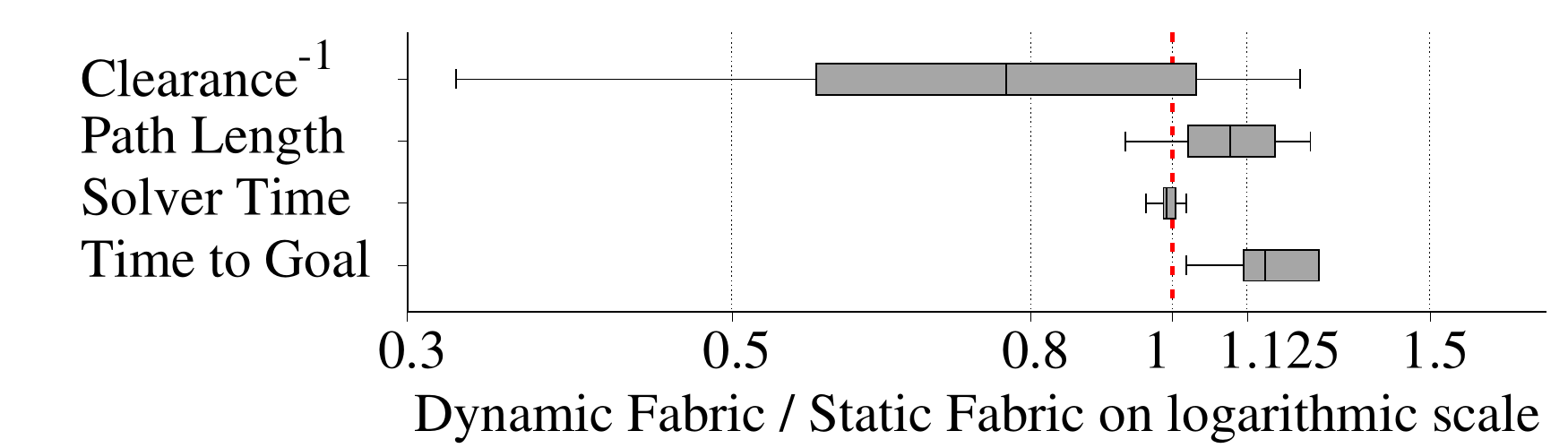}
    \caption{Metrics evaluation for successful experiments}%
    \label{subfig:experiment3_realPanda_res}
  \end{subfigure}
  \begin{subfigure}{1.0\linewidth}
    \centering
    \includegraphics[angle=-90,width=\textwidth]{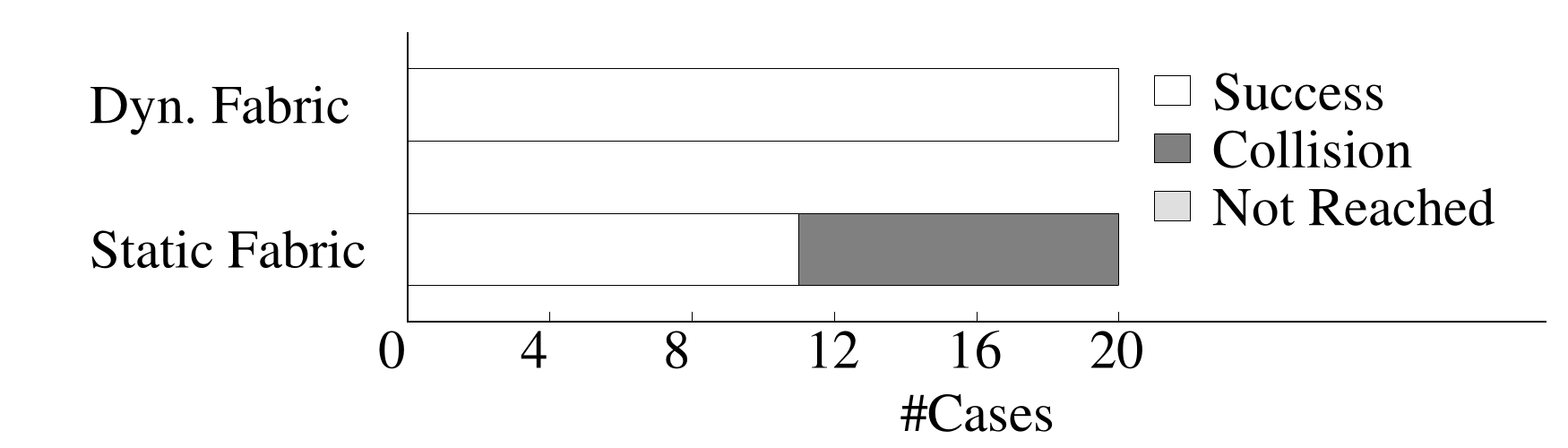}
    \caption{Success results}%
    \label{subfig:experiment3_realPanda_success}
  \end{subfigure}
  \caption{Comparison between \ac{sf} and \ac{df} for real-world scenarios with dynamic obstacles.
  }%
  \label{fig:experiment3_realPanda}
\end{figure}

By investigating one example out of the series, see trajectories in
\cref{fig:experiment3_realPanda_example}, the reason for the large number of collisions
with \ac{sf} can be explained.
Both methods initially drive
the end-effector to the goal position. As the moving obstacle is approaching the robot,
the \ac{df} are starting to react while \ac{sf} are not changing its behavior resulting
in a very sudden motion at around $t=30$s.
\ac{sf} treat moving obstacles as pseudo-static (i.e., the position of the obstacle
is updated at every time step, but the information on its velocity is discarded).  As a
result, the relative velocity between obstacle and robot is only a function of the
velocity of the robot. Geometries and energies for collision avoidance with fabrics are,
by design, a function of this velocity and therefore fail to avoid moving obstacles when
the robot moves slowly or not at all. This behavior is most visible when the goal has
already been reached but an obstacle is approaching. \ac{df} on the other hand
take the velocity of the moving obstacles into account and can therefore avoid them.

\begin{figure}[ht]
  \centering
  \begin{subfigure}{0.5\linewidth}
    \centering
    \includegraphics[width=1.\textwidth]{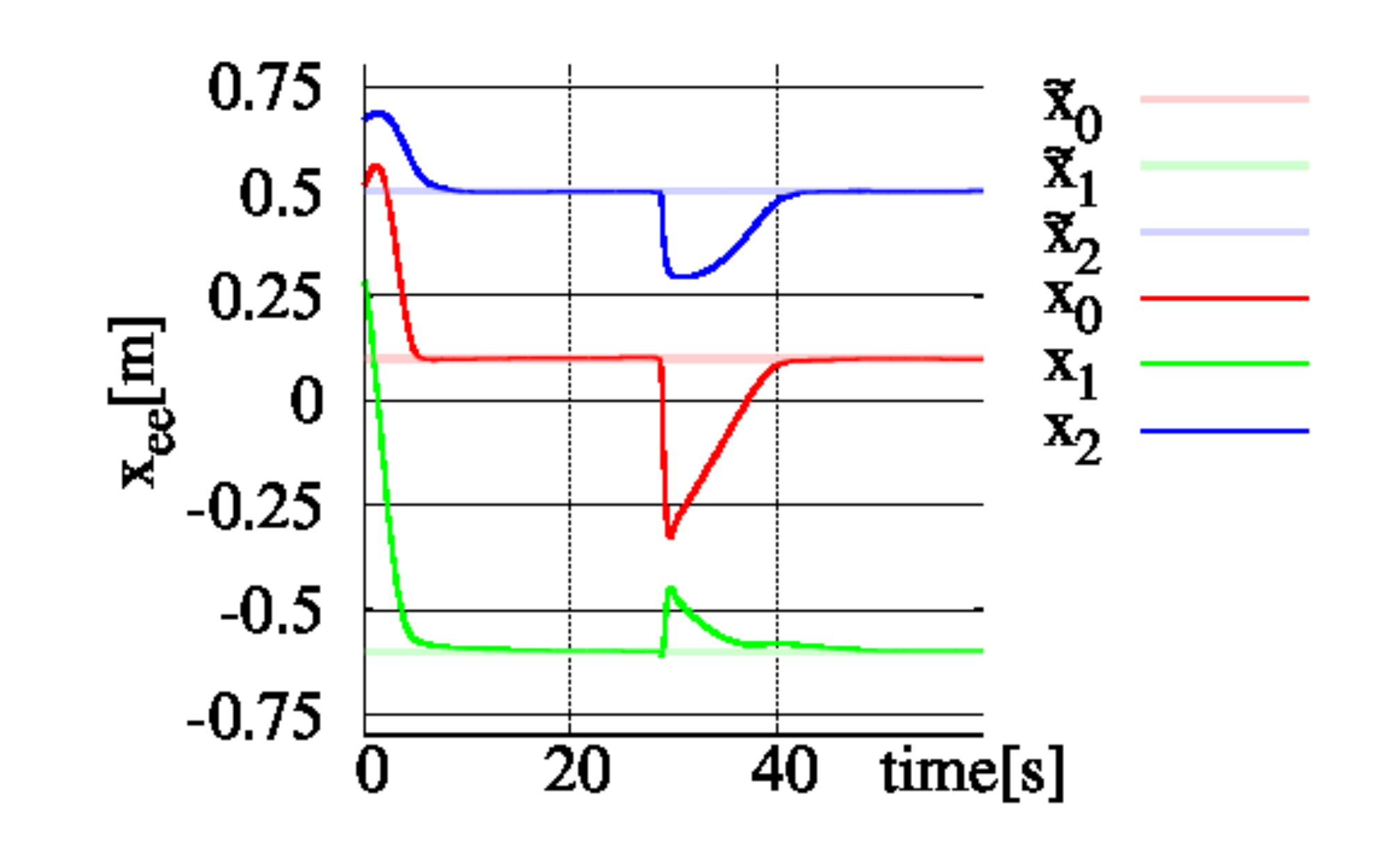}
    \caption{Static Fabric}%
    \label{subfig:experiment3_realPanda_trajectory_static}
  \end{subfigure}%
  \begin{subfigure}{0.5\linewidth}
    \centering
    \includegraphics[width=1.\textwidth]{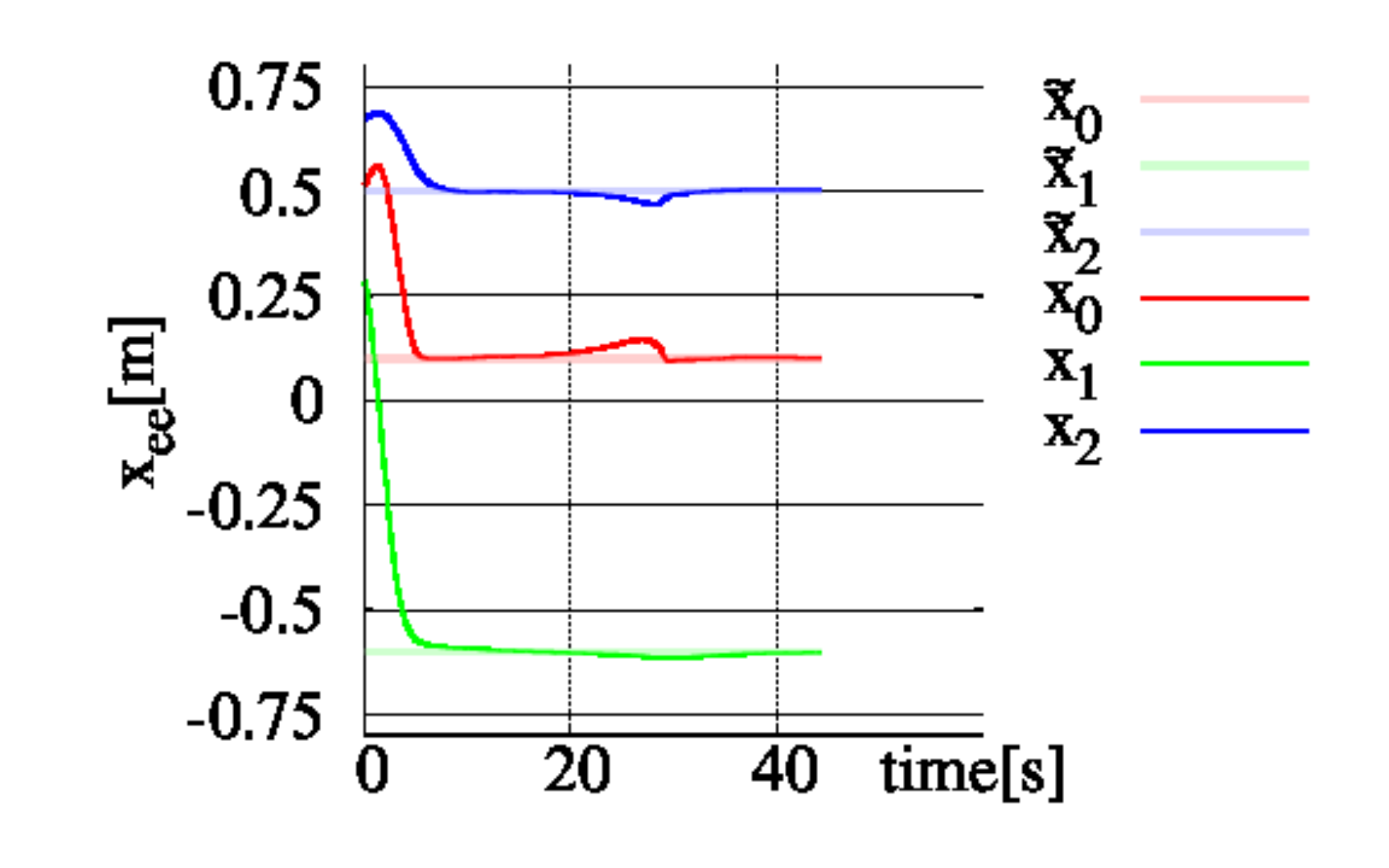}
    \caption{Dynamic Fabric}%
    \label{subfig:experiment3_realPanda_trajectory_dynamic}
  \end{subfigure}
  \caption{Trajectories for real panda robot in the presence of a dynamic obstacle. \ac{df}
    show a smoother and in-advance reaction to the approaching obstacle
    while \ac{sf} can only react in sudden motion.
  }%
  \label{fig:experiment3_realPanda_example}
\end{figure}

%% file: src/results_exp4.tex
\subsection{Experiment 4: Nonholonomic robots}%
\label{sub:experiment_4_nonholonomic_robots}

\paragraph{Simulation}
This experiment assesses the performance of the proposed method to compute
trajectories for non-holonomic robots with fabrics. Specifically, we run
experiments for a \textit{Clearpath Boxer} for position. As
for the first experiment, we compare the performance to \ac{mpc}. In this
experiment, the initial position, the goal location, and the position of five
obstacles were randomized. The results reveal that our extension of
optimization fabrics to non-holonomic robots maintains similar results as with
a robotic arm. Specifically, computational time can be reduced to
optimization-based methods while maintaining good performance in terms of
safety and goal-reaching, \cref{fig:experiment4_simBoxer}. We can also observe
that success rate with \ac{sf} is lower compare to \ac{mpc} due to a high
number of unreached goals.

\begin{figure}[h]
  \centering
  \begin{subfigure}{1.0\linewidth}
    \centering
    \includegraphics[angle=-90,width=\textwidth]{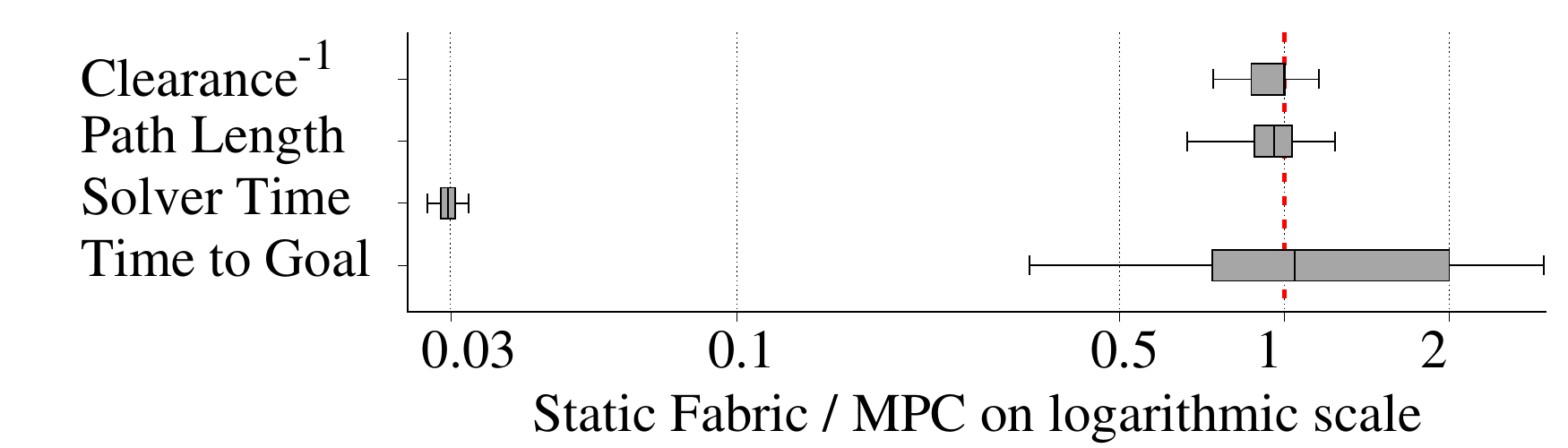}
    \caption{Metrics evaluation for successful experiments}%
    \label{subfig:experiment4_simBoxer_res}
  \end{subfigure}
  \begin{subfigure}{1.0\linewidth}
    \centering
    \includegraphics[angle=-90,width=\textwidth]{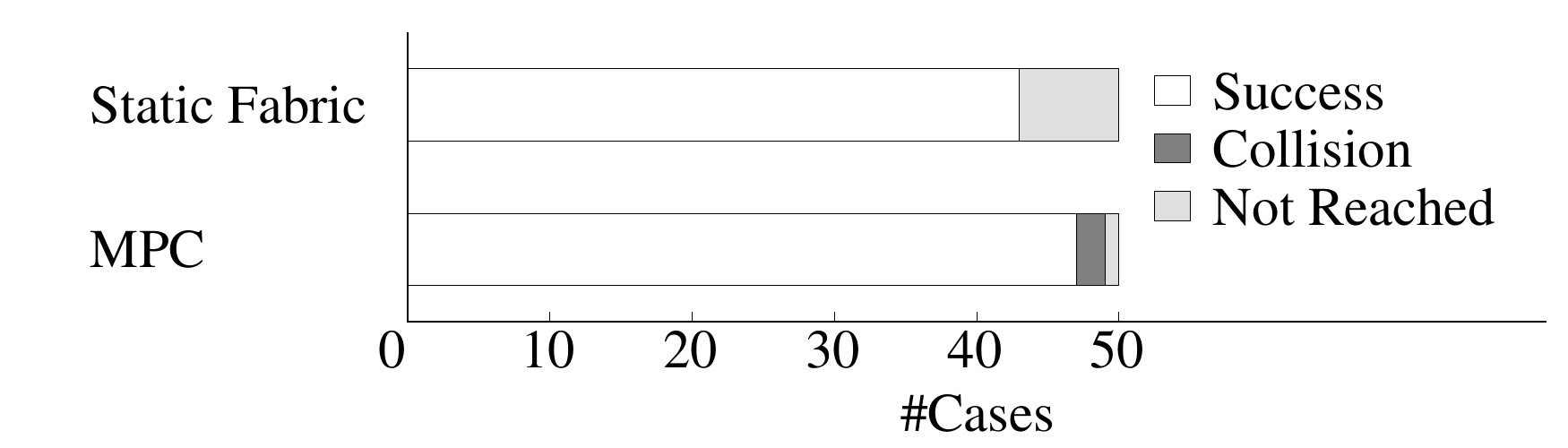}
    \caption{Success results}%
    \label{subfig:experiment4_simBoxer_success}
  \end{subfigure}
  \caption{Results for randomized cased with the Clearpath Boxer robot.
    Similar performance in terms of safety and goal-reaching can be combined with very
    fast computation using optimization fabrics.
  }%
  \label{fig:experiment4_simBoxer}
\end{figure}

%% file: src/results_exp5.tex
\subsection{Experiment 5: Mobile manipulators}%
\label{sub:experiment_5_mobile_manipulators}

In the final experiment, we assess the applicability of \ac{sf} and \ac{df} to a 
non-holonomic mobile manipulator. In an environment that is densely occluded by obstacles,
the motion planning problem is defined by a desired end-effector position and additional
path constraints (e.g. desired orientation of the end-effector).
\paragraph{Simulation}
In simulation, we evaluate the performance of our extension to non-holonomic
mobile manipulators with \ac{sf}. In this series, the positions of 8 obstacles
are randomized for $N=50$ cases. The workspace was limited to a
7mx7m square, so that random obstacles are ensured to be actually hindering the
motion planner. The results reveal that properties shown in the previous
experiments transfer to more complex systems without loss of the computational
benefit, \cref{fig:experiment5_simAlbert_results}. In this series, there were
1 unreached goals and 4 collisions which are, similar
to the previous experiments, caused by local minima. Local minima are more
likely for mobile manipulators as their workspace is larger.
\begin{figure}[h]
  \centering
  \begin{subfigure}{1.0\linewidth}
    \centering
    \includegraphics[angle=-90,width=\textwidth]{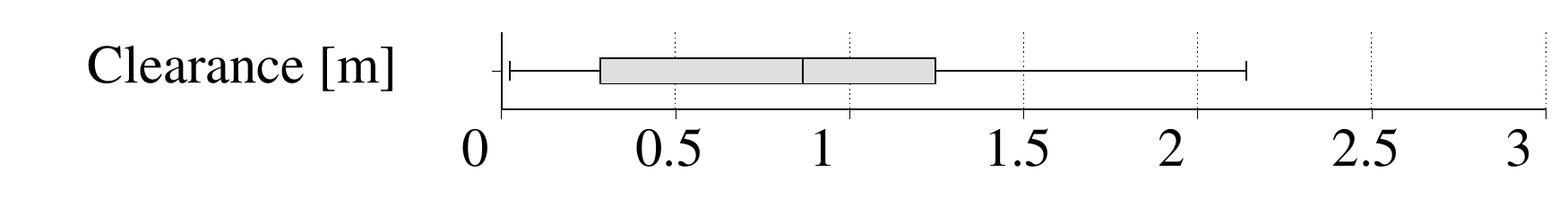}
  \end{subfigure}
  \begin{subfigure}{1.0\linewidth}
    \centering
    \includegraphics[angle=-90,width=\textwidth]{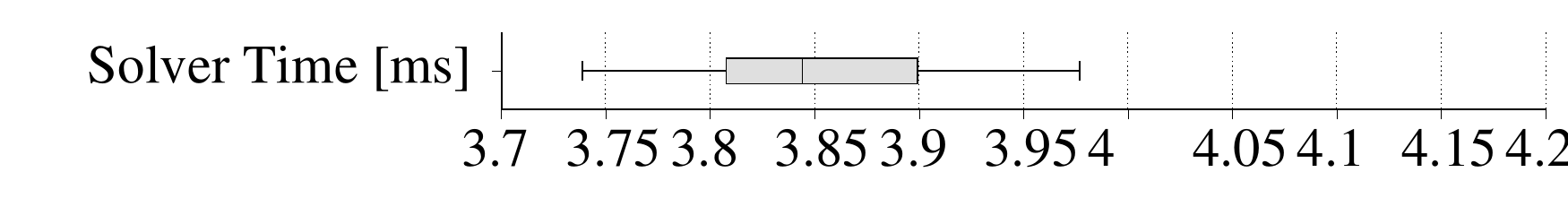}
  \end{subfigure}
  \begin{subfigure}{1.0\linewidth}
    \centering
    \includegraphics[angle=-90,width=\textwidth]{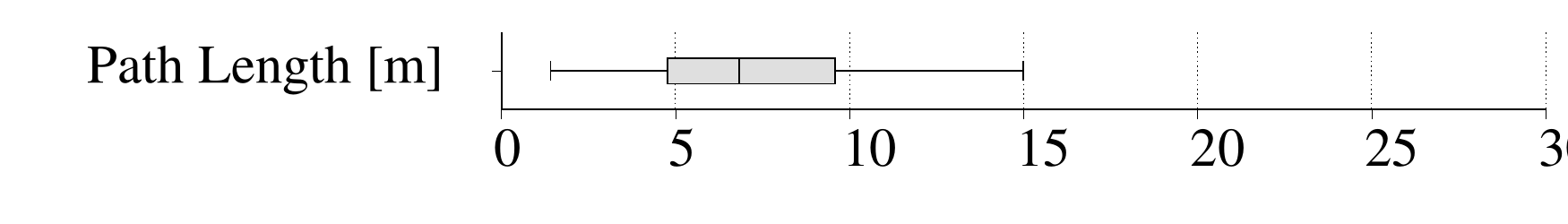}
  \end{subfigure}
  \begin{subfigure}{1.0\linewidth}
    \centering
    \includegraphics[angle=-90,width=\textwidth]{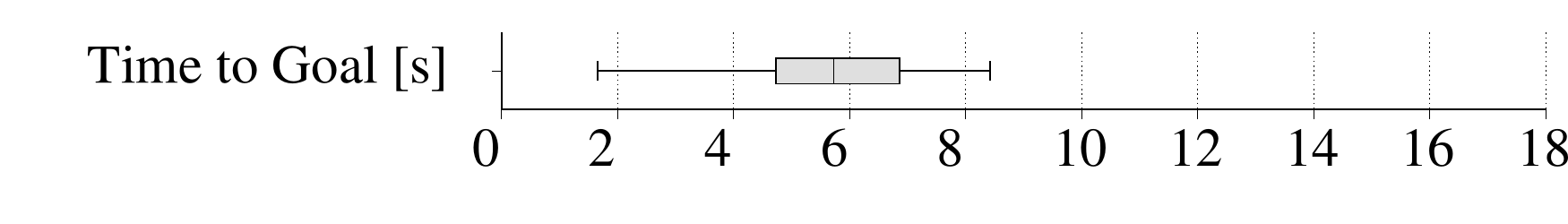}
  \end{subfigure}
  \begin{subfigure}{1.0\linewidth}
    \centering
    \includegraphics[angle=-90,width=\textwidth]{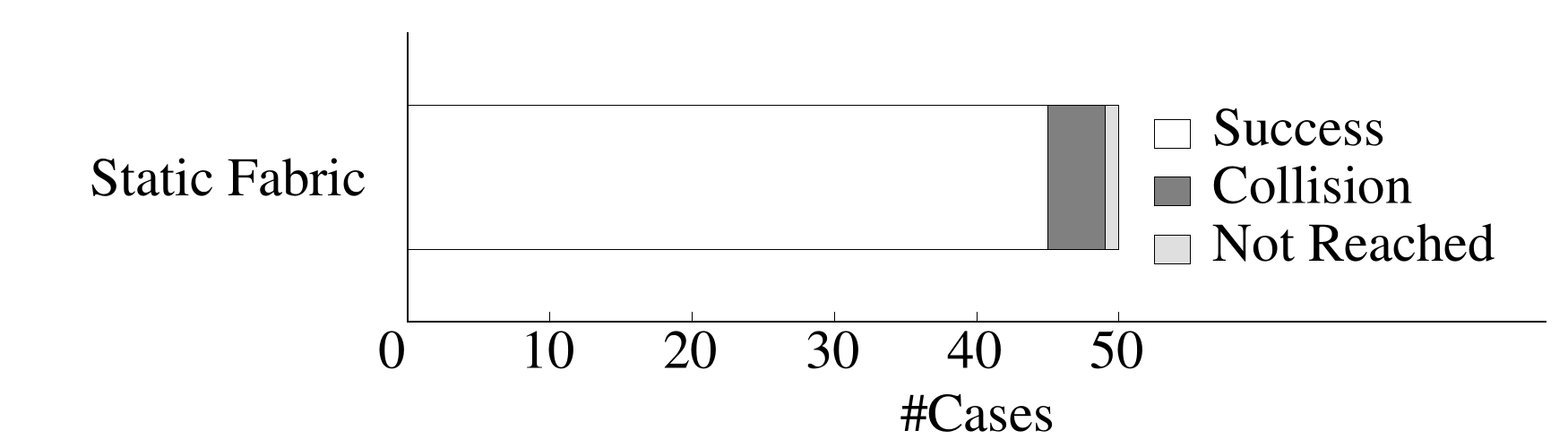}
  \end{subfigure}
  \caption{Quantitative results with static fabrics for a 
    non-holonomic mobile manipulator in simulation. Fabrics solve planning problems in
    randomized environments in low planning time. This allows whole-body control and highly
    reactive behavior.
  }%
  \label{fig:experiment5_simAlbert_results}
\end{figure}
Combining our contributions, \ac{df} and the extension to non-holonomic robots, we
achieve reactive and safe behavior in dynamic environments. Moving obstacles are avoided in a natural way
using our method, see \cref{fig:albert_moving_obstacles}.
\begin{figure*}
  \centering
  \begin{subfigure}{.2\linewidth}
    \centering
    \includegraphics[width=0.9\linewidth]{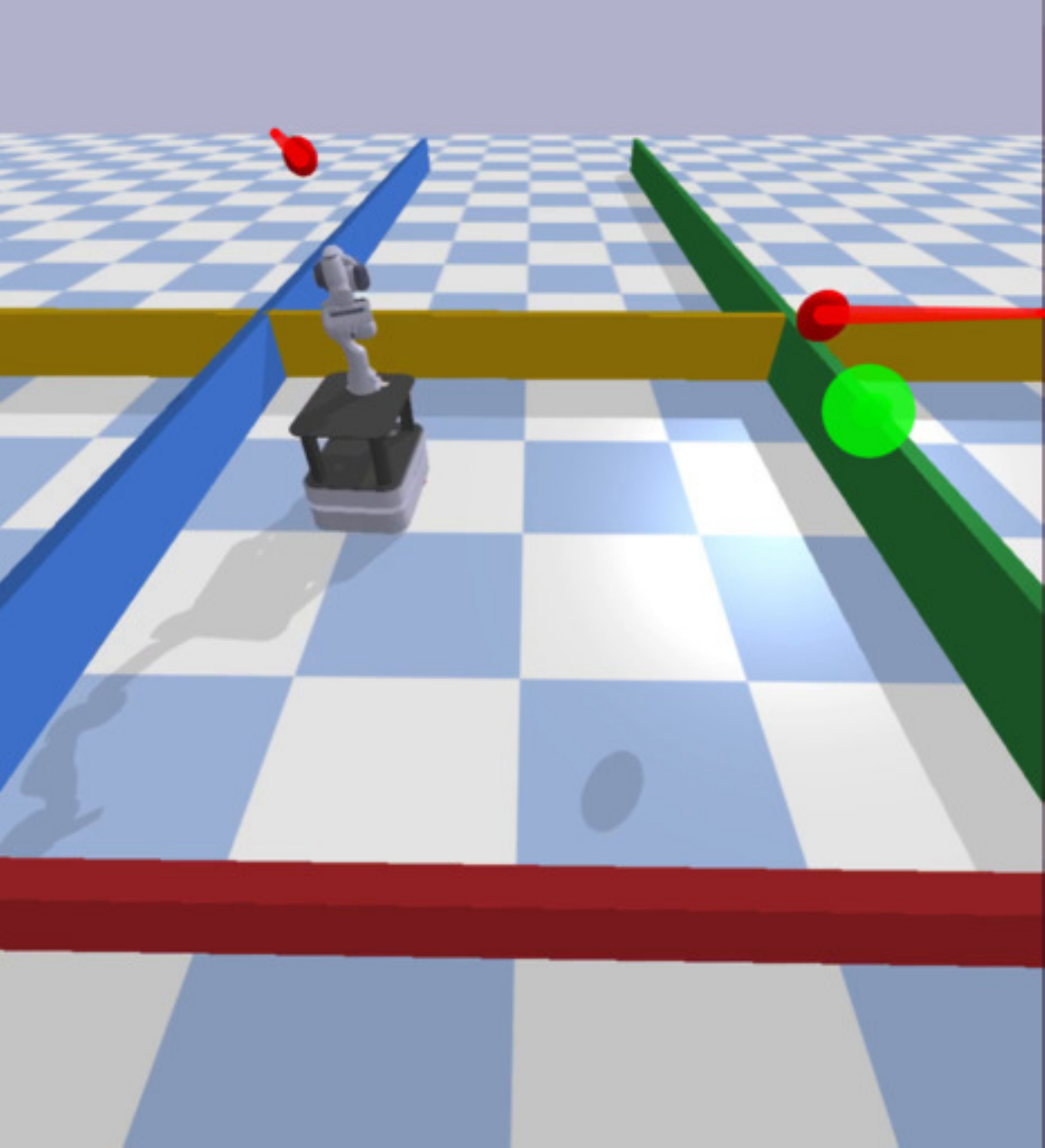}
  \caption{$t=0$s}
  \end{subfigure}%
  \begin{subfigure}{.2\linewidth}
    \centering
    \includegraphics[width=0.9\linewidth]{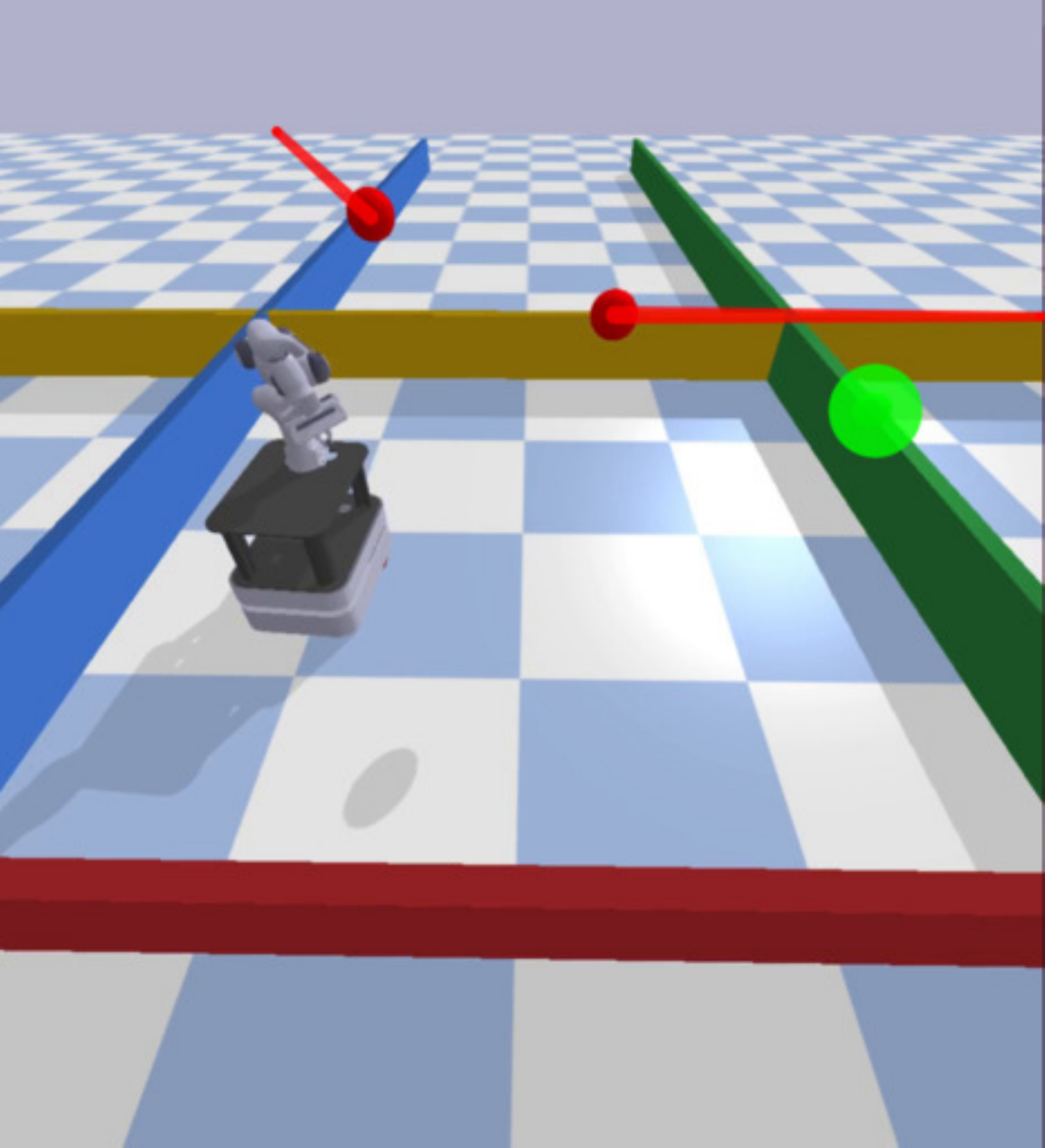}
  \caption{$t=7$s}
  \end{subfigure}%
  \begin{subfigure}{.2\linewidth}
    \centering
    \includegraphics[width=0.9\linewidth]{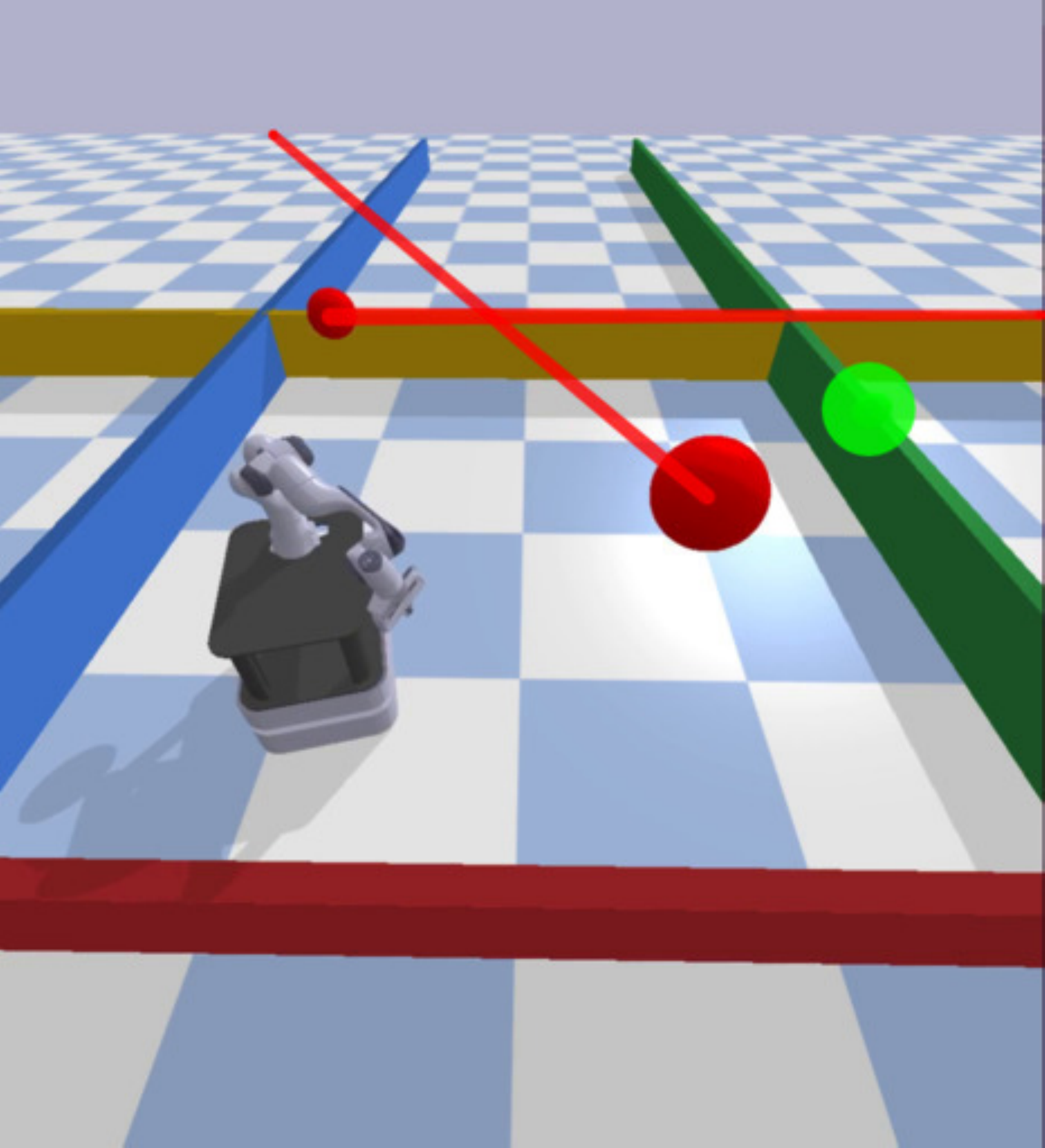}
  \caption{$t=14$s}
  \end{subfigure}%
  \begin{subfigure}{.2\linewidth}
    \centering
    \includegraphics[width=0.9\linewidth]{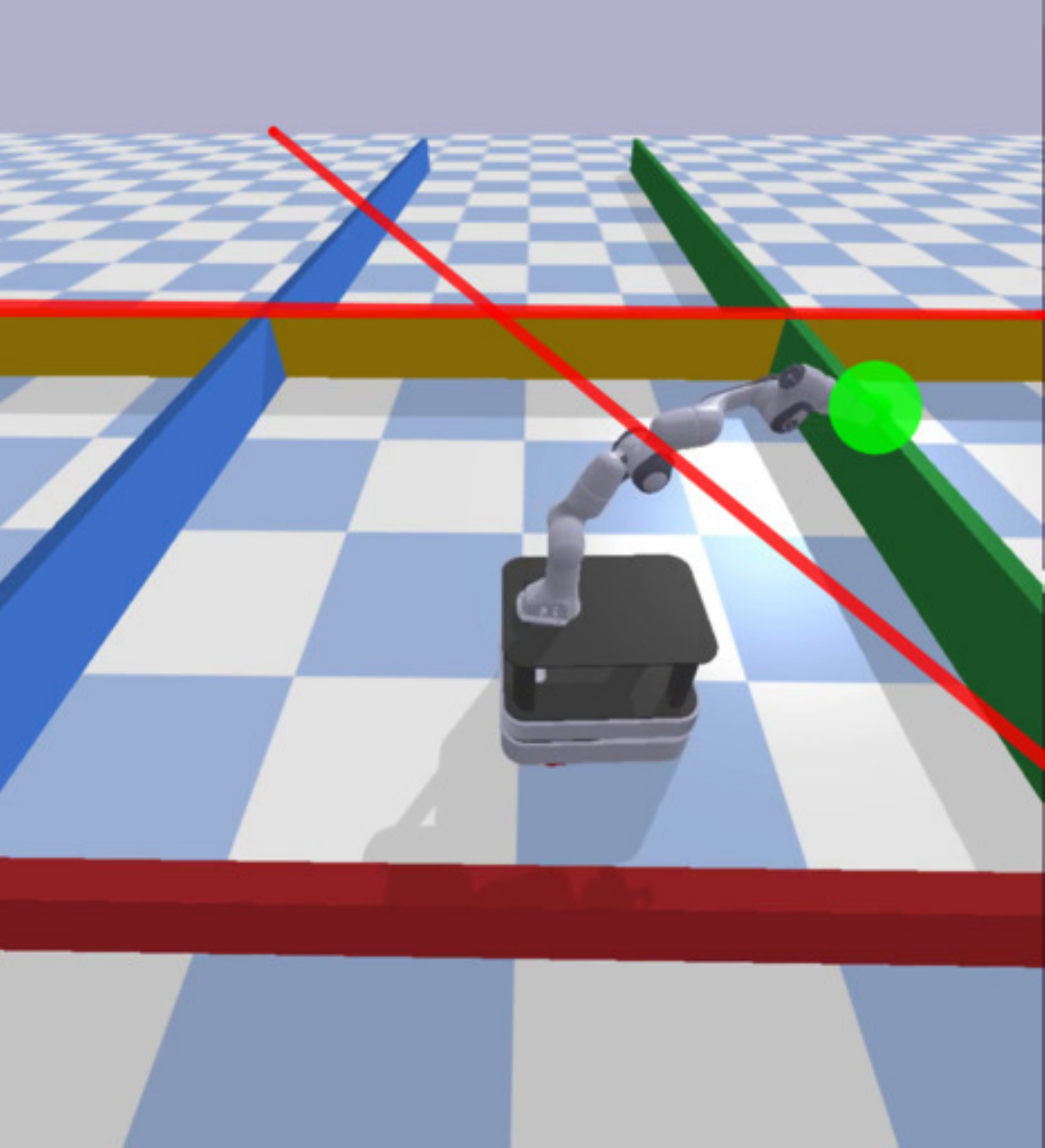}
  \caption{$t=20$s}
  \end{subfigure}%
  \begin{subfigure}{.2\linewidth}
    \centering
    \includegraphics[width=1.0\linewidth]{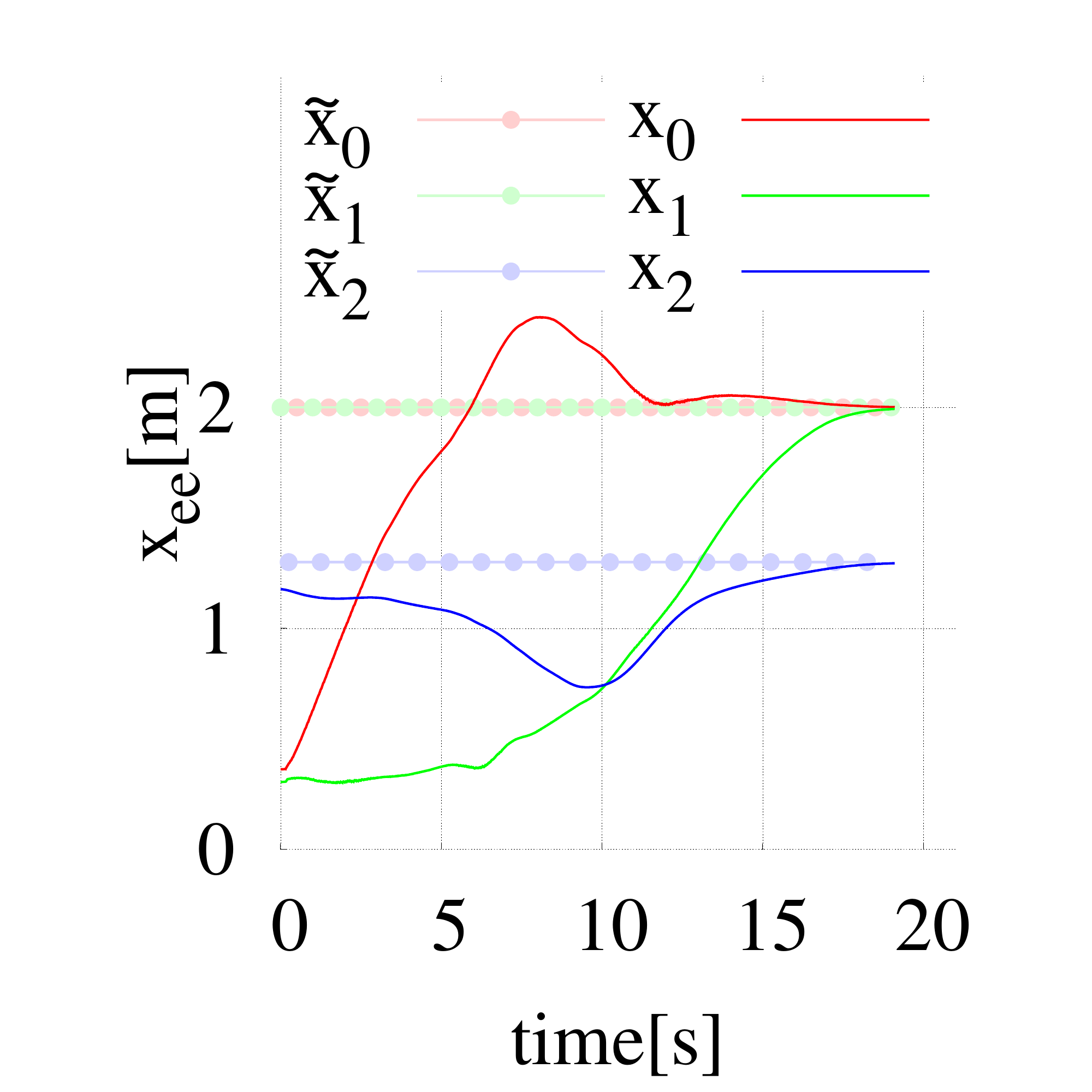}
  \caption{}%
  \label{subfig:albert_moving_obstacles}
  \end{subfigure}
  \caption{Sequence of trajectory computed with \ac{df} for a mobile manipulator in simulation with moving obstacles (red sphere with line indicating the past trajectory) and one end-effector goal (green). The trajectory of the end-effector are 
  visualized in (e) as \x{} and the desired end-effector
  position as \xt{}.
  }%
  \label{fig:albert_moving_obstacles}
\end{figure*}
\paragraph{Real-World}
We present qualitative results for a non-holonomic mobile manipulator using \ac{df}.
In \cref{fig:albert_spline_example}, the robot follows a trajectory defined by a basic spline,
while additionally respecting an orientation constraints on its end-effector and avoiding
the shelves and an obstacle on the ground. The end-effector trajectory is plotted in \cref{fig:albert_spline_trajectory}.
\begin{figure}[t!]
    \centering
    \includegraphics[width=0.9\linewidth]{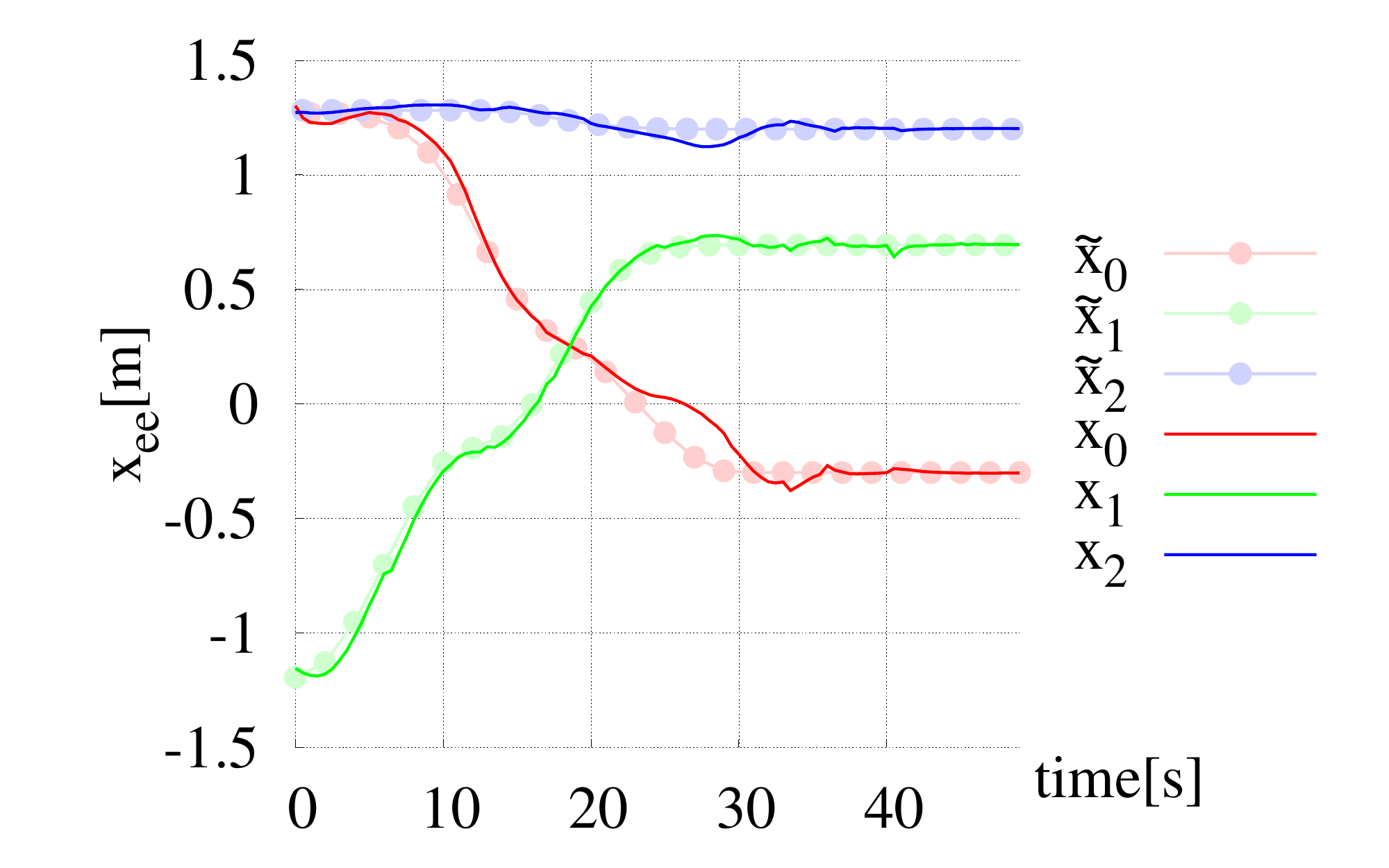}
    \caption{Real-world experiment for path following with a mobile manipulator. The global path can be tracked
        accurately by \ac{df} including the extension to non-holonomic robots. The scene 
        is visualized in \cref{fig:albert_spline_example}.}
    \label{fig:albert_spline_trajectory}
\end{figure}
%

%% file: src/results_exp6.tex
\subsection{Experiment 6: Dynamic fabrics in supermarkets}%
\label{sub:experemint_6_dynamic_fabrics_in_supermarkets}
In this experiment, we show qualitatively how \ac{df} could be used in
collaborative environments where humans and robots coexist. For this
experiment, we give the robot a static goal pose similar to a pickup setup. The
same environment is shared with a co-worker who restocks a shelf. The right
hand of the human is tracked with a motion capture system.
The hand is then avoided by the robot using \ac{df}, see
Fig. 19. In this experiment, the minimum distance
between the robot and the hand was $0.062$m. This real-world experiment
showcases potential applications of the proposed method.

%% file: src/conclusion.tex
\section{Conclusion}%
\label{sec:conclusion}

In this paper, we have generalized optimization fabrics to dynamic
environments. We have proven that our proposed \acl{df} are convergent to
reference paths and can thus compute motion for path following tasks
(\cref{lem:dynamically_energized_fabrics}). Besides, we have proposed an
extension to optimization fabrics (and thus also \ac{df}) for
non-holonomic robots. This allows the application of this framework to a wider
range of robotic applications and ultimately allows the deployment to many
mobile manipulators in dynamic environments.

These theoretical findings were confirmed in various experiments. First, the
quantitative comparisons showed that \acl{sf} outperforms \ac{mpc} in terms of
solver time while maintaining similar performance in terms of goal-reaching and
success rate. 
The improved performance with optimization fabrics might be caused
by the different metric for goal
reaching compared to \acl{mpc}. An integration of
non-Riemannian metrics into an MPC formulation should be further
investigated in the future.

Verifying our theoretical derivations for \ac{df}, the experiments
showed that the deviation error for path following tasks is decreased compared
to \ac{sf}. Similarly, environments with moving obstacles and humans showed
increased clearance while maintaining low computational costs and execution
times. Thus, \ac{df} overcome an important drawback of
\ac{sf}~\cite{Ratliff2020,Wyk2022}, where collision avoidance with moving
obstacle is solved purely by the high frequency at which optimization fabrics
can be computed. Moreover, the generalization did not increase the solving time
compared to \ac{sf}. Unlike the original work on optimization fabrics, this
generalization allows the deployment to dynamic environments where velocity
estimates of moving obstacles are available.

Direct sensor integration in optimization fabrics might be feasible in future
works to overcome the shortcomings of perception pipelines for collision
avoidance. For the trajectory path tasks in this paper, we
used a simple global path generated in workspace. As \ac{df} integrate
global path in arbitrary manifolds, improving the global planning phase could
be further investigated. We expect this to be beneficial when robotics tasks
are constantly changing and task planning is required.

%% file: src/bibio.tex
\changed{
\begin{IEEEbiography}[
  {\includegraphics[width=1in,height=1.25in,clip,keepaspectratio]{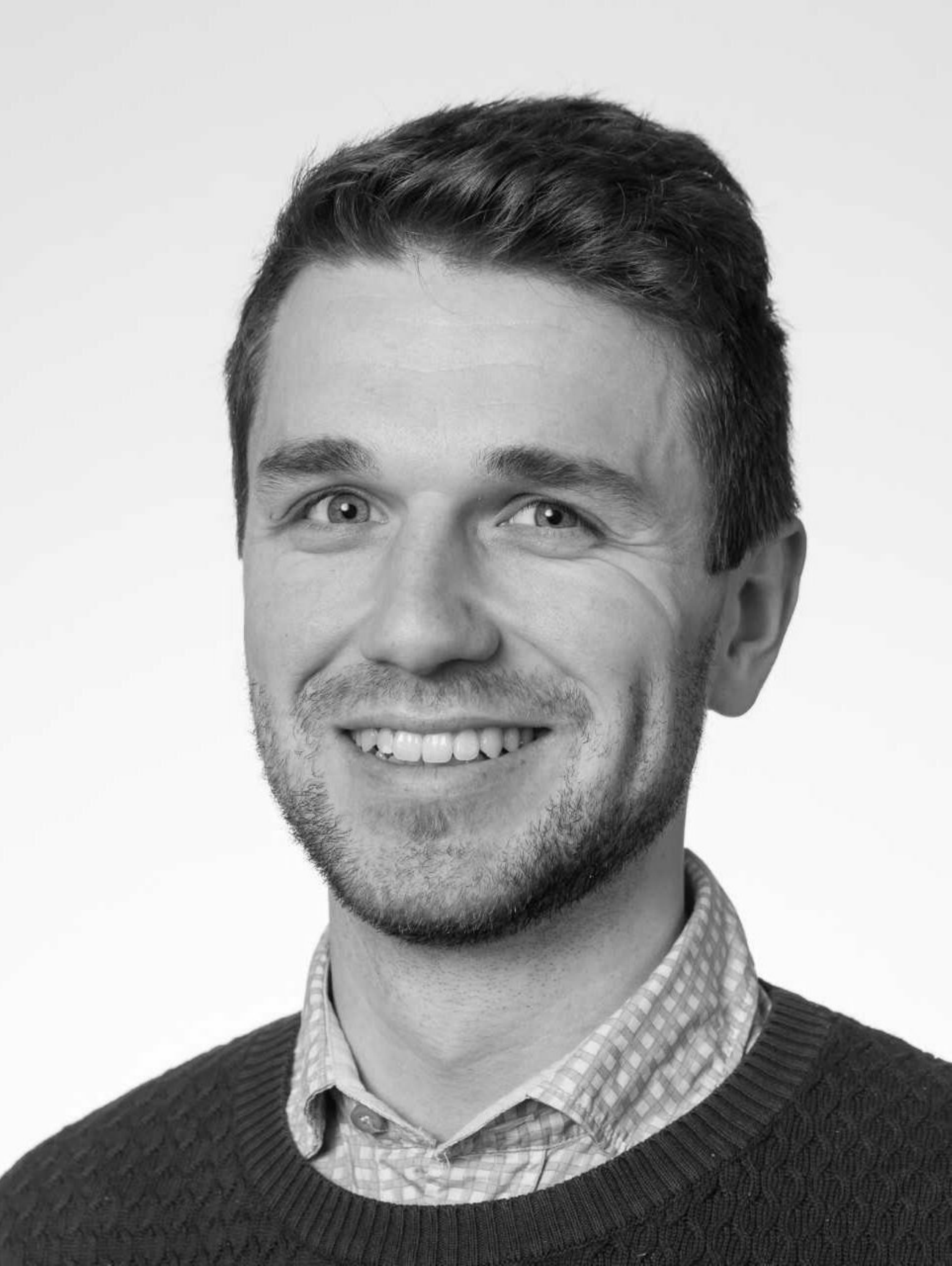}}]
  {Max Spahn} received the B.Sc. degree in Mechanical Engineering at the RWTH
  Aachen in 2018. He received his M.Sc. in General Mechanical Engineering at
  the RWTH in 2019. He also received the degree Diploma d`ingenieur de Centrale
  Supelec in 2020 as part of a double degree program.
 He is currently working towards his Ph.D. degree in robotics. Furthermore, he
  is part of AIRLab, the AI for Retail Lab in Delft. His research interests
  in geometric approaches to motion planning and trajectory generation. He is
  passionate about open-sourcing scientific code, as he believes it will improve
  scientific outcome in robotics.
\end{IEEEbiography}
\vspace{-15mm}
\begin{IEEEbiography}[{\includegraphics[width=1in,height=1.25in,clip,keepaspectratio]{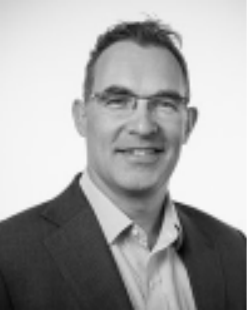}}]{Martijn Wisse}
  received the M.Sc. and Ph.D. degrees in mechanical engineering from the Delft
  University of Technology, Delft, The Netherlands, in 2000 and 2004,
  respectively. He is currently a Professor at the Delft University of Technology.
  His previous research focused on passive dynamic walking robots and passive
  stability in the field of robot manipulators. He worked on underactuated
  grasping, open-loop stable manipulator control, design of robotic systems, and
  the creation of startups in this field. His current research interests focus on
  the neuroscientific principle of active inference and its application and
  advancements in robotics. 
\end{IEEEbiography}
\vspace{-15mm}
\begin{IEEEbiography}[{\includegraphics[width=1in,height=1.25in,clip,keepaspectratio]{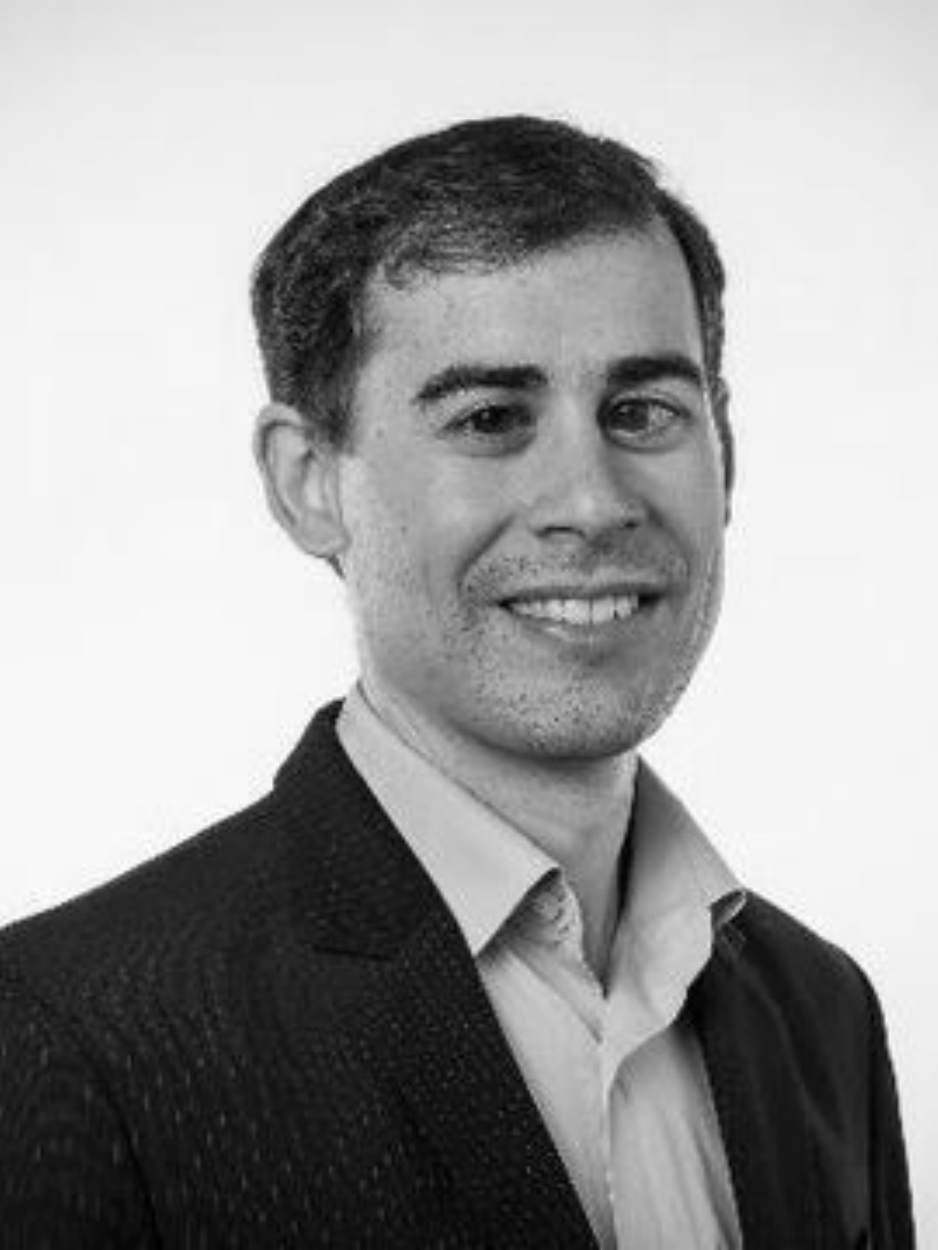}}]{Javier Alonso-Mora}
  is an Associate Professor at the Delft University of
  Technology, where he leads the Autonomous Multi-robots Lab. He received his
  Ph.D. degree in robotics from ETH Zurich, in partnership with Disney Research
  Zurich, and he was a Postdoctoral Associate at the Computer Science and
  Artificial Intelligence Lab (CSAIL) of the Massachusetts Institute of
  Technology. His research focuses on navigation, motion planning, and control of
  autonomous mobile robots, with a special emphasis on multi-robot systems, mobile
  manipulation, on-demand transportation, and robots that interact with other
  robots and humans in dynamic and uncertain environments. He currently serves as
  associate editor for IEEE Transactions on Robotics and for Springer Autonomous
  Robots. He is the recipient of a talent scheme VENI award from the Netherlands
  Organization for Scientific Research (2017), the ICRA Best Paper Award on
  Multi-robot Systems (2019) and an ERC Starting Grant (2021).
\end{IEEEbiography}
\vspace{27em}
}